\documentclass[twoside,11pt]{article}

\usepackage[preprint]{jmlr2e}

\usepackage[ruled,vlined,linesnumbered]{algorithm2e}
\usepackage[noend]{algorithmic}
\usepackage{scalerel} 
\usepackage{amsmath,amsfonts,mathtools}
\usepackage{centernot}
\usepackage{multirow}
\usepackage[inline]{enumitem}
\usepackage{tikz, subcaption}
\usepackage{mathrsfs}
\usetikzlibrary{shapes, arrows, positioning}
\usepackage{array}

\newcolumntype{M}[1]{>{\centering\arraybackslash}m{#1}}
\newcolumntype{N}{@{}m{0pt}@{}}

\newcommand{\independent}{\perp\mkern-9.5mu\perp}
\newcommand{\notindependent}{\centernot{\independent}}
\newcommand{\Mb}{\textit{Mb}}

\DeclarePairedDelimiter{\floor}{\lfloor}{\rfloor}

\firstpageno{1}

\begin{document}

\title{A Recursive Markov Boundary-Based Approach to Causal Structure Learning}

\author{\name Ehsan Mokhtarian\thanks{Equal contribution.} \email ehsan.mokhtarian@epfl.ch \\
       \addr Department of Computer and Communication Science\\
       EPFL, Lausanne, Switzerland
       \AND
       \name Sina Akbari\footnotemark[1] \email sina.akbari@epfl.ch \\
       \addr Department of Computer and Communication Science\\
       EPFL, Lausanne, Switzerland
       \AND
       \name AmirEmad Ghassami \email aghassa1@jhu.edu \\
       \addr Department of Computer Science\\
       Johns Hopkins University, Baltimore, USA
       \AND
       \name Negar Kiyavash \email negar.kiyavash@epfl.ch \\
       \addr College of Management of Technology \\
       EPFL, Lausanne, Switzerland}
\maketitle

\begin{abstract}
    Constraint-based methods are one of the main approaches for causal structure learning that are particularly valued as they are asymptotically guaranteed to find a structure that is Markov equivalent to the causal graph of the system.
    On the other hand, they may require an exponentially large number of conditional independence (CI) tests in the number of variables of the system. 
    In this paper, we propose a novel recursive constraint-based method for causal structure learning that significantly reduces the required number of CI tests compared to the existing literature. 
    The idea of the proposed approach is to use Markov boundary information to identify a specific variable that can be removed from the set of variables without affecting the statistical dependencies among the other variables. Having identified such a variable, we discover its neighborhood, remove that variable from the set of variables, and recursively learn the causal structure over the remaining variables. 
    We further provide a lower bound on the number of CI tests required by any constraint-based method. Comparing this lower bound to our achievable bound demonstrates the efficiency of the proposed approach. 
    Our experimental results show that the proposed algorithm outperforms state-of-the-art both on synthetic and real-world structures.
\end{abstract}


\section{Introduction}\label{sec: intro}
    Learning the causal structure among the variables of the system under study is one of the main goals in many fields of science. 
    This task has also gained significant attention in the recent three decades in artificial intelligence because it has become evident that the knowledge of causal structure can significantly improve the prediction power and remove systematic biases in inference \citep{pearl2009causality,spirtes2000causation}.
    
    One of the main assumptions in causal structure learning is that the ground truth structure is a directed acyclic graph (DAG). 
    There are two main classes of methods in the literature for learning the causal DAG, namely constraint-based methods and score-based methods \citep{zhang2017learning}. 
    The idea in constraint-based methods is to find the most consistent structure with the conditional independence relations in the data. 
    The most well-known constraint-based method is the PC algorithm \citep{spirtes2000causation}. 
    In score-based causal structure learning, the idea is to search for a structure that maximizes a score function, commonly chosen to be a regularized likelihood function. 
    The search for the optimum structure is usually performed via a greedy search \citep{heckerman1995learning,chickering2002optimal,teyssier2012ordering, solus2017consistency}. 
    There are also hybrid methods that combine constraint-based and score-based methods \citep{tsamardinos2006max}, as well as methods that require specific assumptions on the data generating modules, such as requiring linearity and non-Gaussianity of the noises \citep{shimizu2006linear} or additivity of the noise with specific types of non-linearity \citep{hoyer2009nonlinear}.
    
    Constraint-based methods are particularly valued as they do not require any assumptions on the functional form of the causal modules and recover a structure that is Markov equivalent to the causal graph of the system.
    One of the caveats to these methods is that in the worst case, the number of conditional independence (CI) tests required can be exponentially large in the number of variables. 
    Several efforts in the literature have tried to reduce the number of required CI tests and improve the performance of constraint-based methods \citep{margaritis1999bayesian,kalisch2007estimating, pellet2008using,xie2008recursive, zhang2019recursively}.
    
    \citep{spirtes2000causation} proposed the so-called PC algorithm with complexity $\mathcal{O}(p^\Delta)$, where $p$ and $\Delta$ denote the number of variables and the maximum degree of the underlying graph, respectively. 
    \citep{margaritis1999bayesian} and \citep{pellet2008using} proposed using Markov boundary information to reduce the number of required CI tests. 
    The former proposed the GS method with complexity $\mathcal{O}(p^2+ p \alpha^2 2^\alpha)$, and the latter proposed the CS method with complexity $\mathcal{O}(p^22^\alpha)$, where $\alpha$ is the maximum size of the Markov boundary among the variables. 
    To the best of our knowledge, these are the state-of-the-art achievable bounds in the literature.
    
    In this paper, we propose a novel recursive constraint-based method for causal structure learning, which we call \emph{MARVEL}.
    Our method is non-parametric and does not posit any assumptions on the functional relationships among the variables, while it significantly reduces the number of required CI tests.
    In each iteration of our recursive approach, we use the Markov boundary information to find a \emph{removable} variable (see Definition \ref{def: removable}). 
    We then orient the edges incident to this variable and remove it from the set of variables. 
    Finally, we update the Markov boundary information for the next iteration.
    
    Our main contributions are as follows.
    \begin{itemize}
        \item 
            We introduce the notion of a \emph{removable} variable, which is a variable that can be removed from a DAG without changing the d-separation relations (Definition \ref{def: removable}).
            Moreover, we provide a graphical characterization of removable variables (Theorem \ref{thm: removablity}). 
        \item 
            Harnessing the notion of removability, we propose a novel recursive Markov boundary-based causal structure learning method, called MARVEL (Section \ref{sec: MARVEL}). 
            Given the Markov boundary information, MARVEL requires $\mathcal{O}(p \Delta_\text{in}^2 2^{\Delta_\text{in}})$  CI tests in the worst case to find the Markov equivalence class of the causal graph, where $p$ and $\Delta_\text{in}$ are the number of variables and the maximum in-degree of the causal DAG, respectively (Proposition \ref{prp: complexity}). 
            We show that this upper bound significantly improves over the state of the art.
        \item 
            We provide a lower bound on the required number of CI tests for any constraint-based method. 
            Specifically, we show that any constraint-based algorithm requires at least $\Omega(p^2 + p \Delta_\text{in} 2^{\Delta_\text{in}})$ CI tests in the worst case (Theorem \ref{thm: lwrBound}). 
            Comparing this lower bound with our achievable bound demonstrates the efficiency of our proposed method.
    \end{itemize}
    In Section \ref{sec: Markov boundary} we discuss well-known Markov boundary discovery algorithms which require $\mathcal{O}(p^2)$ CI tests. By utilizing one of these algorithms, our method discovers the causal graph by performing $\mathcal{O}(p^2 + p \Delta_\text{in}^2 2^{\Delta_\text{in}})$ CI tests in the worst case.
    It is noteworthy that our upper bound is based on $\Delta_\text{in}$ as opposed to $\Delta$ or $\alpha$. 
    We achieve this by the virtue of recursive variable elimination. 
    In general $\Delta_\text{in} \leq \Delta \leq \alpha$. 
    Additionally, in a DAG with a constant $\Delta_\text{in}$, the values of $\Delta $ and $\alpha$ can grow linearly with the number of variables. 
    Figure \ref{fig: example} depicts one such graph where $\Delta_\text{in}=1$, but $\Delta = \alpha = p-1$. 
    Therefore, in some cases PC, GS, and CS requires an exponential number of CI tests while our approach has merely quadratic complexity. Our experiments on both synthetic and real-world structures show that MARVEL requires substantially fewer CI tests with smaller average size of conditioning sets while obtaining superior accuracy, compared to state-of-the-art constraint-based methods. 
    
	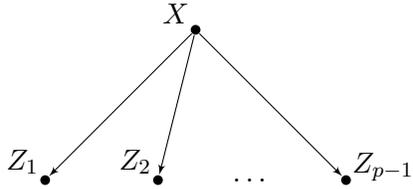
\begin{figure}[t] 
	    \centering
		\tikzstyle{block} = [circle, inner sep=1.3pt, fill=black]
		\tikzstyle{input} = [coordinate]
		\tikzstyle{output} = [coordinate]
        \begin{tikzpicture}
            \tikzset{edge/.style = {->,> = latex'}}
            \node[block] (x) at  (2,2) {};
            \node[] ()[above left=-0.1cm and -0.1cm of x]{$X$};
            \node[block] (z1) at  (0,0) {};
            \node[] ()[above left=-0.1cm and -0.1cm of z1]{$Z_1$};
            \node[block] (z2) at  (1.5,0) {};
            \node[] ()[above left=-0.1cm and -0.1cm of z2]{$Z_2$};
            \node[block] (zp) at  (4,0) {};
            \node[] ()[above right=-0.2cm and -0.1cm of zp]{$Z_{p-1}$};
            \draw[edge] (x) to (z1);
            \draw[edge] (x) to (z2);
            \draw[edge] (x) to (zp);
            \path (z2) to node {\dots} (zp);
        \end{tikzpicture}
        \caption{A DAG with $\Delta_\text{in}=1$ and $\Delta = \alpha = p-1$.}
    \end{figure} \label{fig: example}
    
    We start the exposition by introducing the notations, reviewing the terminology and describing the problem in Section \ref{sec: pre}. MARVEL method is described in Section \ref{sec: MARVEL}, and its computational complexity is discussed in Section \ref{sec: complexity}. Section \ref{sec: experiments} is dedicated to evaluating MARVEL on both synthetic and real-world structures.
    
\section{Preliminaries and Problem Description}\label{sec: pre}
    We consider a system with $p$ variables denoted by the set $\mathbf{V}$. 
    Let $\mathcal{G}=(\mathbf{V},\mathbf{E})$ be the directed acyclic graph (DAG) over $\mathbf{V}$ which represents the causal relationships among the variables, where $\mathbf{E}$ is the set of directed edges. 
    A directed edge from variable $X$ to $Y$, denoted by $(X,Y)$, represents that $X$ is a direct cause of $Y$ with respect to $\mathbf{V}$\footnote{Through out the paper, we use the terms variable and vertex interchangeably.}. 
    This model is referred to as causal DAG or causal Bayesian network in the literature \citep{pearl2009causality, spirtes2000causation, neapolitan2004learning}. 
    The \emph{skeleton} of $\mathcal{G}$ is defined as the undirected graph obtained by removing the directions of the edges of $\mathcal{G}$. 
    If $(X,Y)\in E$, $X$ and $Y$ are called \emph{neighbors} of each other, $X$ is a \emph{parent} of $Y$ and $Y$ is a \emph{child} of $X$. 
    The set of all neighbors, parents, and children of $X$ are denoted by $N_X$, $\text{Pa}_X$, and $\text{Ch}_X$, respectively. 
    \begin{definition}[v-structure]
        Three vertices form a \emph{v-structure} if two of them are parents of the third vertex while they are not neighbors themselves. Additionally, $\mathcal{V}_X^{\text{Pa}}$ denotes the set of v-structures in which $X$ is a parent. 
    \end{definition}
	\begin{definition}[co-parent]
		For $X,Y\in \mathbf{V}$, $Y$ is a co-parent of $X$ if it shares at least one child with $X$ and $Y\not\in N_X$. The set of co-parents of $X$ is denoted by $\Lambda_X$.
	\end{definition} 
	A distribution $P_\mathbf{V}$ on variables $\mathbf{V}$ satisfies \emph{Markov property} with respect to $\mathcal{G}$ if d-separation\footnote{See \citep{pearl1988probabilistic} for the definition of d-separation.} in $\mathcal{G}$ implies conditional independence (CI) in $P_V$. 
	That is, $X$ is d-separated from $Y$ by $\mathbf{S}\subseteq\mathbf{V}$, denoted by $X\perp_{\mathcal{G}} Y\vert \mathbf{S}$, implies $X\independent_{P_\mathbf{V}} Y\vert \mathbf{S}$.
	Conversely, $P_\mathbf{V}$ satisfies \emph{faithfulness} with respect to $\mathcal{G}$, if CI in $P_\mathbf{V}$ implies d-separation in $\mathcal{G}$. 
	That is, $X\independent_{P_\mathbf{V}} Y\vert \mathbf{S}$ implies $X\perp_{\mathcal{G}} Y\vert \mathbf{S}$ \citep{spirtes2000causation, glymour1999computation}. 
	We often drop the subscripts $\mathcal{G}$ and $P_\mathbf{V}$ when there is no ambiguity.
	
	In this paper we study the problem of causal structure learning under the Markov condition and faithfulness assumption. 
	Additionally, we assume causal sufficiency, that is, we assume that the variables do not have any latent common causes. 
	Under these assumptions, the underlying causal DAG can be learned up to Markov equivalence class, i.e., the set of DAGs representing the same conditional independence relationships \citep{spirtes2000causation, pearl2009causality}. 
	\citep{verma1991equivalence} showed that two DAGs are Markov equivalent if and only if they have the same skeleton and v-structures. 
	The Markov equivalence class of a DAG can be graphically represented by a partially directed graph called the \emph{essential graph}. 
	Our goal is to obtain the essential graph corresponding to the causal DAG from observational data. 
	
	Before proceeding to our proposed approach, we briefly review a few definitions and results on Markov boundaries.
	
	\subsection{Markov Boundary}\label{sec: Markov boundary}
	    For $X \in \mathbf{V}$, Markov boundary of $X$ is a minimal set $\mathbf{S} \subseteq \mathbf{V}\setminus\{X\}$ such that $X\independent \mathbf{V}\setminus(\mathbf{S}\cup\{X\}) \vert \mathbf{S}$. 
	    Under Markov and faithfulness assumptions, Markov boundary of each vertex $X$, denoted by $\Mb_X$, is unique and consists of its parents, children and co-parents \citep{pearl1988probabilistic, pearl2009causality}:
    	\begin{equation}
    	    \Mb_X= \text{Pa}_X \cup \text{Ch}_X \cup \Lambda_X = N_X \cup \Lambda_X.
    	\end{equation}
    	Many algorithms have been proposed in the literature for discovering the Markov boundaries \citep{fu2010markov, margaritis1999bayesian,guyon2002gene, tsamardinos2003towards, tsamardinos2003algorithms,yaramakala2005speculative}. One straightforward method is \emph{total conditioning} (TC)  \citep{pellet2008using}, which states that under faithfulness, $X$ and $Y$ are in each other's Markov boundary if and only if
		\begin{equation}
		    X \notindependent Y \vert \mathbf{V} \setminus \{X,Y\}.
		\end{equation}
	    Using total conditioning, $\binom{|\mathbf{V}|}{2}$ CI tests suffice for identifying the Markov boundaries of all of the vertices. 
	    However, the drawback is that each CI test requires conditioning on a large set of variables. 
	    
	    This issue is addressed in several algorithms including Grow-Shrink (GS) \citep{margaritis1999bayesian}, IAMB \citep{tsamardinos2003algorithms} and its several variants which propose a method that leads to performing more CI tests, but with smaller conditioning sets. Thus, the choice of which algorithm to use for computing the Markov boundaries must be made according to the data at hand.
        Note that these algorithms perform at most $\mathcal{O}(p^2)$ CI tests to discover the Markov boundaries. The Markov boundary information is required to initialize our proposed approach, and any of these methods can be utilized for this purpose.

\section{MARVEL Method}\label{sec: MARVEL}
    In this section, we present our recursive method for learning the causal structure. 
    The idea in this approach is as follows. 
    
    We first identify a variable with certain properties, which we call \emph{removable} using Markov boundary information. A removable variable is a variable that can be omitted from the causal graph such that the remaining graph still satisfies Markov property and faithfulness with respect to the marginal distribution of the remaining variables.
    We then identify the neighbors of this variable and orient the edges incident to it. 
    Finally, we remove this variable from the set of variables and update the Markov boundaries for the next iteration.
    This procedure is repeated until all of the variables are removed.
    
    We show that this approach leads to finding a graph with the same skeleton and v-structures as the true causal graph.
    Hence, the corresponding essential graph can be identified using this graph.
    We call our method, \emph{Markov boundary-based Recursive Variable Elimination (MARVEL)}.
    
    We introduce removable variables in Subsection \ref{sec: removable variables}, propose a method for testing removability in Subsection \ref{sec: testing removability}, and present our recursive algorithm in Subsection \ref{sec: algo}. 
    In Subsection \ref{sec: save} we show how to avoid performing duplicate CI tests during our algorithm. 
    
    \subsection{Removable Variables} \label{sec: removable variables}
        A removable variable is defined formally as follows.
        \begin{definition}[Removable] \label{def: removable}
        	$X$ is a removable vertex in a DAG $\mathcal{G}$ if 
        	the d-separation relations in $\mathcal{G}$ and $\mathcal{H} := \mathcal{G} \setminus \{X\}$ are equivalent over the vertices of $\mathcal{H}$.
        	That is for any vertices $Y,Z\in\mathbf{V}\setminus\{X\}$ and $\mathbf{S}\subseteq\mathbf{V}\setminus\{X,Y,Z\}$,
            \begin{equation}\label{eq: d-sepEquivalence}
    		    Y \perp_{\mathcal{G}} Z \vert \mathbf{S}
    		    \iff
    		    Y \perp_{\mathcal{H}} Z \vert \mathbf{S}.
    		\end{equation}
        \end{definition}
        \begin{remark} \label{remark: removable}
            Suppose $P_\mathbf{V}$ is Markov and faithful with respect to a DAG $\mathcal{G}$. 
            For any vertex $X\in \mathbf{V}$, $P_{\mathbf{V}\setminus\{X\}}$ is Markov and faithful with respect to $\mathcal{G}\setminus\{X\}$ if and only if $X$ is a removable vertex in $\mathcal{G}$. 
        \end{remark}
        All the proofs are available in Appendix \ref{sec: appendix}. Next, we propose a graphical characterization of removable variables.
        \begin{theorem}[Removability] \label{thm: removablity}
    		$X$ is removable in $\mathcal{G}$ if and only if the following two conditions are satisfied for every $Z\in\text{Ch}_{X}.$
    		\begin{description}
    			\item 
    			    Condition 1: $N_X\subset N_Z\cup\{Z\}.$
    			\item 
    			    Condition 2: $\text{Pa}_Y \subset \text{Pa}_Z$ for any $Y\in\text{Ch}_X\cap \text{Pa}_Z$.
    		\end{description}
    	\end{theorem}
    	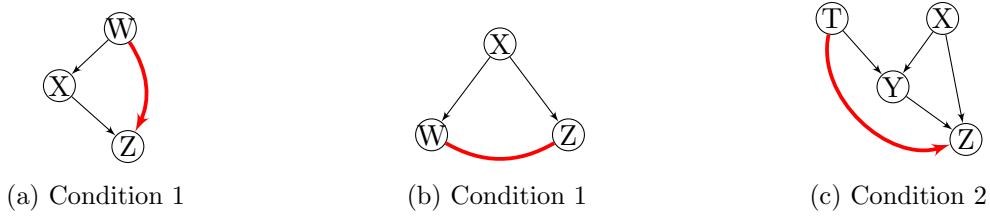
\begin{figure}[h] 
    	    \centering
    		\tikzstyle{block} = [draw, fill=white, circle, text centered,inner sep=0.15cm]
    		\tikzstyle{input} = [coordinate]
    		\tikzstyle{output} = [coordinate]
    		\begin{subfigure}[b]{0.3\textwidth}
    			\centering
    			\begin{tikzpicture}[->, auto, node distance=1.3cm,>=latex']
    			\node [block, label=center:X](X) {};
    			\node [block, below right= 0.5cm and 0.6cm of X, label=center:Z](Z) {};
    			\node [block, above right= 0.45cm and 0.5cm of X, label=center:W](W) {};
    			\draw (X) to (Z);
    			\draw (W) to (X);
    			\draw [line width=0.5mm, red] (W) to [bend left] (Z);
    			\end{tikzpicture}
    			\caption{Condition 1}
    			\label{fig:R, a}
    		\end{subfigure}\hfill
    		\begin{subfigure}[b]{0.3\textwidth}
    			\centering
    			\begin{tikzpicture}[->, auto, node distance=1.3cm,>=latex', every node/.style={inner sep=0.12cm}]
    			\node [block, label=center:X](X) {};
    			\node [block, below right= 0.9cm and 0.6cm of X, label=center:Z](Z) {};
    			\node [block, below left= 0.9cm and 0.6cm of X, label=center:W](W) {};
    			\draw (X) to (Z);
    			\draw (X) to (W);
    			\draw [-, line width=0.5mm, red] (W) to [bend right] (Z);
    			\end{tikzpicture}
    			\caption{Condition 1}
    			\label{fig:R, b}
    		\end{subfigure}\hfill
    		\begin{subfigure}[b]{0.3\textwidth}
    			\centering
    			\begin{tikzpicture}[->, auto, node distance=1.3cm,>=latex', every node/.style={inner sep=0.12cm}]
    			\node [block, label=center:X](X) {};
    			\node [block, below right= 1.3cm and 0cm of X, label=center:Z](Z) {};
    			\node [block, below left= 0.6cm and 0.35cm of X, label=center:Y](Y) {};
    			\node [block, above left= 0.6cm and 0.5 cm of Y, label=center:T](U){};
    			\draw (X) to (Y);
    			\draw (X) to (Z);
    			\draw (Y) to (Z);
    			\draw (U) to (Y);
    			\draw [line width=0.5mm, red] (U) to [bend right=60] (Z);
    			\end{tikzpicture}
    			\caption{Condition 2}
    			\label{fig:R, c}
    		\end{subfigure}
    		\caption{Conditions of Theorem \ref{thm: removablity}.}
    		\label{fig:R}
    	\end{figure}
        Figure \ref{fig:R} depicts the two Conditions of Theorem \ref{thm: removablity}.
        In Figures \ref{fig:R, a} and \ref{fig:R, b}, $W\in N_X$. Hence, Condition 1 implies the red edge between $W$ and $Z$.
        Note that $W$ must be a parent of $Z$ in Figure \ref{fig:R, a} since $\mathcal{G}$ is a DAG.
        In Figure \ref{fig:R, c}, $T\in \text{Pa}_Y$. Hence, Condition 2 implies the red edge between $T$ and $Z$.
    	\begin{remark} \label{rem: removables exist}
    	    Variables with no children satisfy the two conditions of Theorem \ref{thm: removablity}. Therefore, there always exists a removable vertex in every DAG.
    	\end{remark}
	\subsection{Testing for removability} \label{sec: testing removability}
	    In this section, we present an approach for testing the removability of a variable given the Markov boundary information of all of the variables when Markov and faithful assumptions hold, i.e., when d-separation is equivalent to CI. 
	    Our method is based on evaluating the two conditions of Theorem \ref{thm: removablity}. 
	    In Subsections \ref{subsub: neighbor} and \ref{subsub: vs} we discuss how to identify the neighbors, co-parents, and $\mathcal{V}_X^{\text{pa}}$ efficiently.
	    Subsequently, we use this information to test Conditions 1 and 2 of Theorem \ref{thm: removablity} in Subsections \ref{subsub: condition1} and \ref{subsub: condition2}. 
	    \subsubsection{Finding Neighbors and co-parents of a variable} \label{subsub: neighbor}
	        Given the Markov boundary of a variable, we can use the following lemma to tell the neighbors and co-parents apart \citep{margaritis1999bayesian}.
         	\begin{lemma} \label{lem: neighbor}
         	    Suppose $X\in \mathbf{V}$ and $Y\in\Mb_X$. 
                $Y$ is a neighbor of $X$ if and only if 
                \begin{equation} \label{eq: neighbor}
                    X\notindependent Y \vert \mathbf{S}, \hspace{ 0.5cm} \forall \mathbf{S}\subsetneq \Mb_X \setminus \{Y\}.
                \end{equation}
         	\end{lemma}
            Given $\Mb_X$, for each $Y\in \Mb_X$ we perform the CI tests of Equation \ref{eq: neighbor}.
            If all of theses tests yield dependence, then $X$ and $Y$ are neighbors.
            Otherwise, there exists $\mathbf{S}_{XY}\subsetneq \Mb_X \setminus \{Y\} $ that d-separates $X$ and $Y$. 
            Hence, we can identify $N_X$ and $\Lambda_X$ (along with a separating set for each co-parent) with at most $|\Mb_X| 2^{|\Mb_X|-1}$ CI tests.
        \subsubsection{Finding $\mathcal{V}_X^{\text{pa}}$} \label{subsub: vs}
            Recall that $\mathcal{V}_X^{\text{pa}}$ denotes the set of v-structures in which $X$ is a parent.
            \begin{lemma}\label{lem: v-structures}
                Suppose $T \in \Lambda_X$ with a separating set $\mathbf{S}_{XT}$ for $X$ and $T$, and let $Y\in N_X$. 
                $Y$ is a common child of $X$ and $T$ (i.e., $X\to Y\gets T$ is in $\mathcal{V}_X^{\text{pa}}$)  if and only if $Y\notin \mathbf{S}_{XT}$ and
                \begin{equation}
                    Y\notindependent T\vert \mathbf{S}, \hspace{0.5cm} \forall \mathbf{S}\subseteq \Mb_X\cup\{X\}\setminus\{Y,T\}.
                \end{equation}
            \end{lemma}
            Once $N_X$ and $\Lambda_X$ are identified along with a separating set for each of the variables in $\Lambda_X$ using Lemma \ref{lem: neighbor}, we can find $\mathcal{V}_X^{\text{pa}}$ by applying Lemma \ref{lem: v-structures} which requires performing at most $|\Lambda_X||N_X|2^{|\Mb_X|-1}$ CI tests.
	    \subsubsection{Testing condition 1} \label{subsub: condition1}
	        To test condition 1 of Theorem \ref{thm: removablity} for a variable $X$, we apply the following lemma. 
	        \begin{lemma} \label{lem: condition1}
        		Variable $X$ satisfies Condition 1 of Theorem \ref{thm: removablity} if and only if 
        		\begin{equation} \label{eq: condition1}
        		    Z \notindependent W\vert \mathbf{S} \cup \{X\}, \hspace{0.5cm} \forall W,Z\in N_X,\, \mathbf{S} \subseteq \Mb_X \setminus \{Z, W\}.
        		\end{equation}
        	\end{lemma}
        	Given $N_X$, Condition 1 for $X$ can be verified using Lemma \ref{lem: condition1} by performing at most $\binom{|N_X|}{2} 2^{|\Mb_X|-2}$ CI tests. 
        \subsubsection{Testing condition 2} \label{subsub: condition2}
	        \begin{lemma} \label{lem: condition2}
	            Suppose the variable $X$ satisfies Condition 1 of Theorem \ref{thm: removablity}. Then $X$ satisfies Condition 2 of Theorem \ref{thm: removablity}, and therefore, $X$ is removable, if and only if 
	            \begin{equation} \label{eq: condition2}
	                Z \notindependent T\vert \mathbf{S} \cup \{X, Y\}, \hspace{0.5cm} \forall (X\to Y\gets T) \in \mathcal{V}_X^{\text{pa}}, \, Z\in N_X \setminus \{Y\},\, \mathbf{S} \subseteq \Mb_X \setminus \{Z, Y,  T\}.
	            \end{equation}
	        \end{lemma}
	        Having identified $N_X$ and $\mathcal{V}_X^{\text{pa}}$, Lemma \ref{lem: condition2} allows us to verify if Condition 2 holds for $X$ using at most $\vert\Lambda_X\vert\vert N_X\vert 2^{|\Mb_X|-2}$ unique CI tests, since $Z\in N_X$, $T\in \Lambda_X$, and $\mathbf{S} \cup \{Y\} \subseteq \Mb_X \setminus \{Z,  T\}$. 
	        In Section \ref{sec: save} we show how to perform these CI tests without performing duplicate tests. 
	        The following proposition summarizes the results of this section. 
	        \begin{proposition} \label{prop: removability}
	            Knowledge of $\Mb_X$ suffices to identify $N_X, \Lambda_X$, $\mathcal{V}_X^{\text{pa}}$, and determine whether $X$ is removable by performing at most $\mathcal{O}(|\Mb_X|^2 2^{|\Mb_X|})$ unique CI tests. 
	        \end{proposition}
    \subsection{Algorithm} \label{sec: algo}
    	In order to identify a removable vertex efficiently (i.e. in terms of the number of CI tests), we sort the variables based on the cardinality of their Markov boundary in ascending order. 
    	Let $\mathcal{I} =(X_{(1)},X_{(2)},...,X_{(p)})$ be this ordering. Starting with $X_{(1)}$, we look for the first variable $X_{(i)}$ that is removable. 
    	Testing the removability of each variable is performed through the CI tests described in the previous section, and we stop when we identify the first removable variable in this order.
    	The following lemma guarantees that the Markov boundary of every variable for which we perform the removability tests until we reach the first one to prove removable is at most of size $\Delta_{\text{in}}$. 
    	\begin{lemma}\label{lem: MbboundForremovable}
    		If $X\in\mathbf{V}$ is a removable vertex in $\mathcal{G}$, then $\left\vert \Mb_X\right\vert\leq \Delta_{\text{in}}$, where $\Delta_{\text{in}}$ is the maximum in-degree of $\mathcal{G}$.
    	\end{lemma}
    	As shown in Section \ref{sec: testing removability}, the CI tests required for both learning the neighbors of a variable and testing its removability are conditioned on a subset of its Markov boundary. Since vertices are processed in the order $\mathcal{I}$ up to a removable variable, by Lemma \ref{lem: MbboundForremovable} we only process the variables with a Markov boundary of size at most $\Delta_{\text{in}}$. This ensures a maximum size of conditioning sets, which in turn results in more powerful conditional independence tests, and also results in performing substantially fewer CI tests to learn the structure.
    	Note that if we recursively eliminate the first removable variable in $\mathcal{I}$ at each iteration, the succeeding variables in $\mathcal{I}$, which are currently not removable, will eventually become removable at some point. 
    	Therefore, during the process, we never need to learn the neighborhood of a variable with Markov boundary size larger than $\Delta_{\text{in}}$. 
    	We discuss our complexity upper bound in Section \ref{sec: complexity}.
    	The pseudo code of MARVEL is outlined in  Algorithm \ref{MARVEL}.
    	
        \begin{algorithm}[tb]
           \caption{MARVEL}
           \label{MARVEL}
        \begin{algorithmic}[1]
            \STATE {\bfseries Input:} $\mathbf{V},\, {P}_\mathbf{V},\, (\Mb_X\!:\:X\in \mathbf{V})$
            \STATE Create $\hat{\mathcal{G}} = (\mathbf{V},\,\mathbf{E}=\varnothing)$ with vertex set $\mathbf{V}$ and no edges.
            \STATE $\overline{\mathbf{V}} \gets \mathbf{V}$ \hfill \% $\overline{V}$ is the set of remaining variables.
		    \FOR{1 {\bfseries to} $|\mathbf{V}|$}
		        \STATE $\mathcal{I} = (X_{(1)},X_{(2)},...,X_{(|\overline{\mathbf{V}}|)}) \gets$ Sort $\overline{\mathbf{V}}$ in ascending order based on their Markov boundary size.
		        \FOR{$i=1$ {\bfseries to} $|\overline{\mathbf{V}}|$}
		            \STATE $(N_{X_{(i)}},\,  \Lambda_{X_{(i)}},\, \mathcal{S}_{X_{(i)}}) \gets$ Find neighbors and co-parents of  $X_{(i)}$ along with a set of separating sets $\mathcal{S}_{X_{(i)}}= (\mathbf{S}_{X_{(i)}Y}\!:\: Y\in \Lambda_{X_{(i)}})$
		            using Lemma \ref{lem: neighbor}. 
		            \STATE Add undirected edges between $X_{(i)}$ and $N_{X_{(i)}}$ to $\hat{\mathcal{G}}$ if an edge is not already there.
		            \IF{Equation \ref{eq: condition1} holds for $X=X_{(i)}$}
		                \STATE Find $\mathcal{V}_{X_{(i)}}^{\text{pa}}$ using Lemma \ref{lem: v-structures}.
		                \STATE Orient the edges of $\mathcal{V}_{X_{(i)}}^{\text{pa}}$ accordingly in $\hat{\mathcal{G}}$.
                        \IF{Equation \ref{eq: condition2} holds for $X=X_{(i)}$}
                            \STATE Orient the remaining undirected edges incident to $X_{(i)}$ in $\hat{\mathcal{G}}$ as in-going towards $X_{(i)}$.
                            \STATE $\overline{\mathbf{V}} \gets \overline{\mathbf{V}}\setminus\{X_{(i)}\}$
                            \STATE Update $\Mb_Y$ for all $Y\in\Mb_{X_{(i)}}$. \hfill \% See Subsection \ref{sec: update Mb}.
                            \STATE \textbf{Break} the for loop of line 6.
		                \ENDIF
		            \ENDIF
		        \ENDFOR
		    \ENDFOR
		    \STATE $\Tilde{\mathcal{G}} \gets$ The partially directed graph with the skeleton and the v-structures of $\hat{\mathcal{G}}$.
		    \STATE Apply the Meek rules \citep{meek1995causal} to $\Tilde{\mathcal{G}}$.
		    \RETURN $\Tilde{\mathcal{G}}$
        \end{algorithmic}
        \end{algorithm}
        
        Initially, the output of a Markov boundary discovery algorithm (e.g., one of the algorithms mentioned in Section \ref{sec: Markov boundary}) is input to the MARVEL method.
    	The procedure is then initialized with an empty graph $\hat{\mathcal{G}}$ (line 2). 
    	$\overline{\mathbf{V}}$ indicates the set of remaining variables.
    	Each iteration of the recursive part consists of three main phases: 
        \begin{enumerate}
        	\item Identify the first removable vertex $X_{(i)}$ in $\mathcal{I}$.
        	\item Discover $N_{X_{(i)}}$ and orient the edges incident to $X_{(i)}$.
        	\item Remove $X_{(i)}$ and update the Markov boundary of the remaining variables.
        \end{enumerate} 
        In phase 1, we first obtain $\mathcal{I}$ in line 5. 
        We then check the two condition of removability using Lemmas \ref{lem: condition1} and \ref{lem: condition2} in lines 9 and 12.
        Phase 2 is performed in the following steps. 
        First, we learn $N_{X_{(i)}}$ and add the corresponding undirected edges to $\hat{\mathcal{G}}$ if the edge is not already added (lines 7, 8). 
        We then orient the v-structure edges of $\mathcal{V}_X^{\text{pa}}$ (line 10). 
        Finally, we orient the remaining undirected edges incident to $X_{(i)}$ as in-going towards $X_{(i)}$ (line 13). 
        We will show in Theorem \ref{thm: correctness} that 
        with this orientation, $\hat{\mathcal{G}}$ will have the same v-structures as the causal graph $\mathcal{G}$. 
        In phase 3, $X_{(i)}$ is removed from the graph, and we update the Markov boundaries of the remaining variables in the absence of $X_{(i)}$ (lines 14, 15). In Subsection \ref{sec: update Mb} we will present an efficient method for the latter task. Note that $X_{(i)}$ is particularly chosen such that the marginal distribution of the remaining variables satisfies faithfulness with respect to the induced subgraph of $\mathcal{G}$ over these variables.

        At the end of the algorithm, we keep the direction of the edges involved in v-structures in $\hat{\mathcal{G}}$ and apply the Meek rules \citep{meek1995causal} to this partially directed graph.
        Note that initially, we assumed Markov property and faithfulness.
        Although we remove a variable at each iteration, Remark \ref{remark: removable} implies that Markov and faithfulness assumptions hold in all iterations since we remove removable variables. 
        Hence, throughout the execution of the algorithm, the CI tests are equivalent to the d-separation relations in the remaining graph.
        
        The theorem below provides the correctness of MARVEL method.
    	\begin{theorem}[Correctness of MARVEL]\label{thm: correctness}
    	    Suppose $\mathcal{G}$ satisfies Markov property and faithfulness with respect to $P_{\mathbf{V}}$.
    	    The learned graph $\hat{\mathcal{G}}$ in Algorithm \ref{MARVEL} has the same skeleton and v-structures as $\mathcal{G}$. 
    	    Therefore, the output of Algorithm \ref{MARVEL} is the essential graph corresponding to $\mathcal{G}$.
    	\end{theorem} 

        \subsubsection{Updating Markov boundaries} \label{sec: update Mb}
        	Suppose $X$ is removed in an iteration. 
        	At the end of this iteration, we need to update the Markov boundaries of the remaining vertices. 
        	The removal of $X$ affects the remaining vertices in two ways:
        	\begin{enumerate}
        		\item $X$ should be removed from Markov boundaries of every vertex it appeared in.
        		\item If two variables are not adjacent and have only $X$ as a common child, then they should be removed from each other's Markov boundary. 
        	\end{enumerate}
        	Consequently, it is sufficient to remove $X$ from the Markov boundary of the vertices in $\Mb_X$, and then update the Markov boundary only for the vertices in $N_X$, i.e., the only vertices that can potentially be parents of $X$.
        	Therefore, it suffices that for all pairs $\{Y,Z\}$ in $N_X$ we check whether $Y$ and $Z$ remain in each other's Markov boundary after removing $X$. 
        	This is equivalent to testing for the dependency $Y\notindependent Z\vert \Mb_Z \setminus \{X,Y,Z\}$, or alternatively $Y\notindependent Z\vert \Mb_Y \setminus \{X,Y,Z\}$. 
        	We perform the test using the smaller of these two conditioning sets. 
        	Formally, define $W\coloneqq\arg\min_{U\in\{Y,Z\}}|\Mb_U|$. 
        	The test will check whether
        	\begin{equation} \label{eq: update Mb}
        	    Y\independent Z\vert \Mb_W \setminus \{X,Y,Z\}.
        	\end{equation}
        	If Equation \ref{eq: update Mb} holds, we remove $Y$ from $\Mb_Z$ and $Z$ from $\Mb_Y$. 
    \subsection{Avoiding duplicate CI tests} \label{sec: save}
        In an iteration of MARVEL, if a variable does not pass the removability tests, it will be tested for removability in the subsequent iterations. 
        We discuss an approach to use the information from previous iterations to avoid performing duplicate CI tests while testing the removability of such variables.
        \subsubsection{CI tests for finding neighbors, co-parents, and v-structures}
            We identify neighbors and co-parents of a variable in line 7 of Algorithm \ref{MARVEL} as described in \ref{subsub: neighbor}. 
            It suffices to do this procedure once for each variable $X$. 
            More precisely, if we learn $N_X$ in an iteration, the neighbors of $X$ in the next iterations would be $N_X$ excluding the deleted variables. 
            Hence, we do not need any further CI tests for finding neighbors and co-parents of $X$ in the following iterations. 
            
            The same applies to v-structures. 
            If we find $\mathcal{V}_X^{\text{pa}}$ in an iteration, we can save it and delete a v-structure from it when one of the three variables of the v-structure is removed.
        \subsubsection{CI tests for condition 1}
            For an arbitrary variable $X$, suppose $Z,W\in N_X$. 
            If in an iteration, there does not exist any $\mathbf{S}\subseteq \Mb_X \setminus \{Z,W\}$ such that $Z \independent W \vert \mathbf{S} \cup \{X\}$, then either $Z,W$ are neighbors or they are both parents of $X$.
            Since removing a vertex does not alter such relationships, no $\mathbf{S} \cup \{X\}$ can separate $W,Z$ in the following iterations either. 
            We can use this information to skip performing duplicate CI tests for pairs $W,Z\in N_X$ in the following iterations.
        \subsubsection{CI tests for condition 2}    
            The same idea applies to condition 2. 
            Suppose $Z\in N_X$ and $X\to Y \gets T$ is in $\mathcal{V}_X^{\text{pa}}$.
            If in an iteration for all $\mathbf{S} \subseteq \Mb_X \setminus \{Z, W\}$, $Z \notindependent T\vert \mathbf{S} \cup \{X, Y\}$, then $Z,T$ do not have any separating set including $X,Y$. 
            Hence, we can save this information and avoid performing CI tests for the pair $Z,T$ when the separating set includes both $X,Y$. 

\section{Complexity Analysis} \label{sec: complexity}
    The bottleneck in the complexity of constraint-based causal structure learning methods is the number of CI tests they perform. 
    In this section, we provide an upper bound on the worst-case complexity of the MARVEL algorithm in terms of the number of CI tests and compare it with PC \citep{spirtes2000causation}, GS \citep{margaritis1999bayesian}, CS \citep{pellet2008using}, and MMPC \citep{tsamardinos2003time} algorithms. 
    We also provide a worst-case lower bound for any constraint-based algorithm to demonstrate the efficiency of our approach.
	\begin{proposition} \label{prp: complexity}
		Given the initial Markov boundaries, the number of CI tests required by Algorithm \ref{MARVEL} on a graph of order $p$ and maximum in-degree $\Delta_{in}$ is upper bounded by
		\begin{equation}\label{eq: complexity}
		    p \binom{\Delta_{\text{in}}}{2} + \frac{p}{2}\Delta_{\text{in}}(1+ 0.45\Delta_{\text{in}} )2^{\Delta_{\text{in}}}  = \mathcal{O}(p\Delta_{in}^22^{\Delta_{in}}).
		\end{equation} 
	\end{proposition}
    \begin{corollary}
        If $\Delta_\text{in} \leq c \log p$, MARVEL uses at most $\mathcal{O}(p^{c+1}\log^2 p)$ CI tests in the worst case, which is polynomial in the number of variables. 
    \end{corollary}
    If we use any of the algorithms mentioned in Section \ref{sec: Markov boundary} that compute the Markov boundaries with $\mathcal{O}(p^2)$ CI tests, the overall upper bound of MARVEL on the number of CI tests will be
    \begin{equation} \label{eq: upper bound}
        \mathcal{O}(p^2 + p\Delta_{in}^22^{\Delta_{in}}).
    \end{equation}
    \begin{theorem} \label{thm: lwrBound}
		The number of conditional independence tests of the form $X \independent Y\vert \mathbf{S}$ required by any constraint-based algorithm on a graph of order $p$ and maximum in-degree $\Delta_{in}$ in the worst case is lower bounded by
		\begin{equation} \label{eq: lwrbound}
		  \Omega(p^2+p\Delta_{in}2^{\Delta_{in}}).
		\end{equation}
	\end{theorem}
    For constant $\Delta_{in}$ or more generally  $\Delta_{in}=o(\log p)$, Equations \ref{eq: upper bound} and \ref{eq: lwrbound} are quadratic in $p$. 
    For larger values of $\Delta_{in}$, the upper bound differs from the lower bound by a factor of $\Delta_{in}$.
    
    \begin{table}[h]
		\centering
		\caption{Number of required CI tests in the worst case by various algorithms for causal structure learning.}
		\begin{tabular}{ N|M{3cm}||M{7cm}| }
			\hline
			 & Algorithm
			&  Number of CI tests in the worst case\\
			\hline
			\hline
			& PC\footnotemark
			& $\mathcal{O}(p^\Delta)$ \\
			\hline
			&GS 
			& $\mathcal{O}(p^2+ p \alpha^2 2^\alpha)$\\
			\hline
			&CS  
			& $\mathcal{O}(p^22^\alpha$)\\
			\hline
			&MMPC 
			& $\mathcal{O}(p^22^\alpha$)\\
			\hline
			&MARVEL  
			& $\mathcal{O}(p^2+p\Delta_{in}^22^{\Delta_{in}})$\\
			\hline
			\hline
			&Lower bound
			& $\Omega(p^2 + p\Delta_{in}2^{\Delta_{in}})$\\
			\hline
		\end{tabular}
		\label{table}
	\end{table}
	\footnotetext{If PC priory knows the exact value of $\Delta_{in}$ as side information, its upper bound  will be $\mathcal{O}(p^{\Delta_{in}})$.}
    Table \ref{table} compares the complexity of various algorithms, where $\Delta$ and $\alpha$ denote the maximum degree and the maximum Markov boundary size of the causal graph, respectively. 
    In general $\Delta_\text{in} \leq \Delta \leq \alpha$. 
    Additionally, in a DAG with a constant in-degree, $\Delta$ and $\alpha$ can grow linearly with the number of variables. 
    Therefore, not only MARVEL has a significantly smaller worst-case bound, but also the complexity bound for all other algorithms can be exponential in some regimes where MARVEL remains polynomial.

\section{Experiments} \label{sec: experiments}
	We evaluate MARVEL and compare it with other methods in two settings\footnote{All of the experiments were run in MATLAB on a MacBook Pro laptop equipped with a 1.7 GHz Quad-Core Intel Core i7 processor and a 16GB, 2133 MHz, LPDDR3 RAM.}. 
	In Subsection \ref{sec: noisless}, we assess the complexity of various causal structure learning algorithms in terms of the number of CI tests and size of conditioning sets, given oracle CI tests, i.e., when algorithms have access to true conditional independence relations among the variables.
	This is similar to assuming that the size of the observed samples is large enough to recover the conditional independence relations without any error. 
	In this case, all of the algorithms recover the essential graph corresponding to the causal graph. 
	In Subsection \ref{sec: noisy}, we evaluate the algorithms on finite sample data, where we compare both the complexity and the accuracy of the algorithms over a wide range of sample sizes.
	
    Our comparison cohort includes the modified version of PC algorithm that starts from the moralized graph\footnote{The moralized graph of a DAG is the undirected graph in which every vertex is connected to all the variables in its Markov boundary.} instead of the complete graph (to make a fair comparison with algorithms that start with Markov boundary information) \citep{pellet2008using, spirtes2000causation}, GS \citep{margaritis1999bayesian}, CS \citep{pellet2008using}, and MMPC \citep{tsamardinos2003time} algorithms. 
    We use MATLAB implementations of PC and MMPC provided in \citep{murphy2001bayes, Murphy2019bayest} and \citep{Tsirilis2018M3HC}, respectively. 
    The implementation of MARVEL is available in https://github.com/Ehsan-Mokhtarian/MARVEL.
    \subsection{Oracle Setting} \label{sec: noisless}
        \begin{figure*}[t]
            \centering
            \captionsetup{justification=centering}
            \begin{subfigure}[b]{\textwidth}
                \centering
                \includegraphics[width=0.55\textwidth]{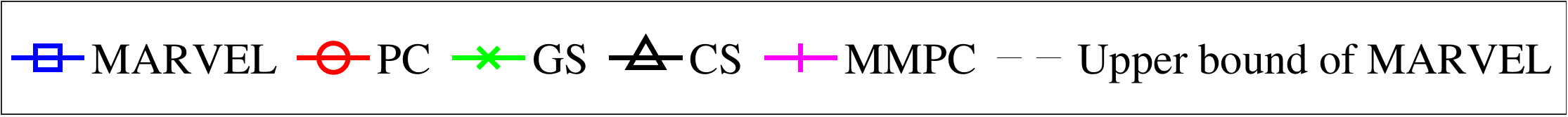}
            \end{subfigure} 
            
            \begin{subfigure}[b]{0.8\textwidth}
                \centering 
                \hspace*{-0.4cm}\includegraphics[width= 0.4 \textwidth]{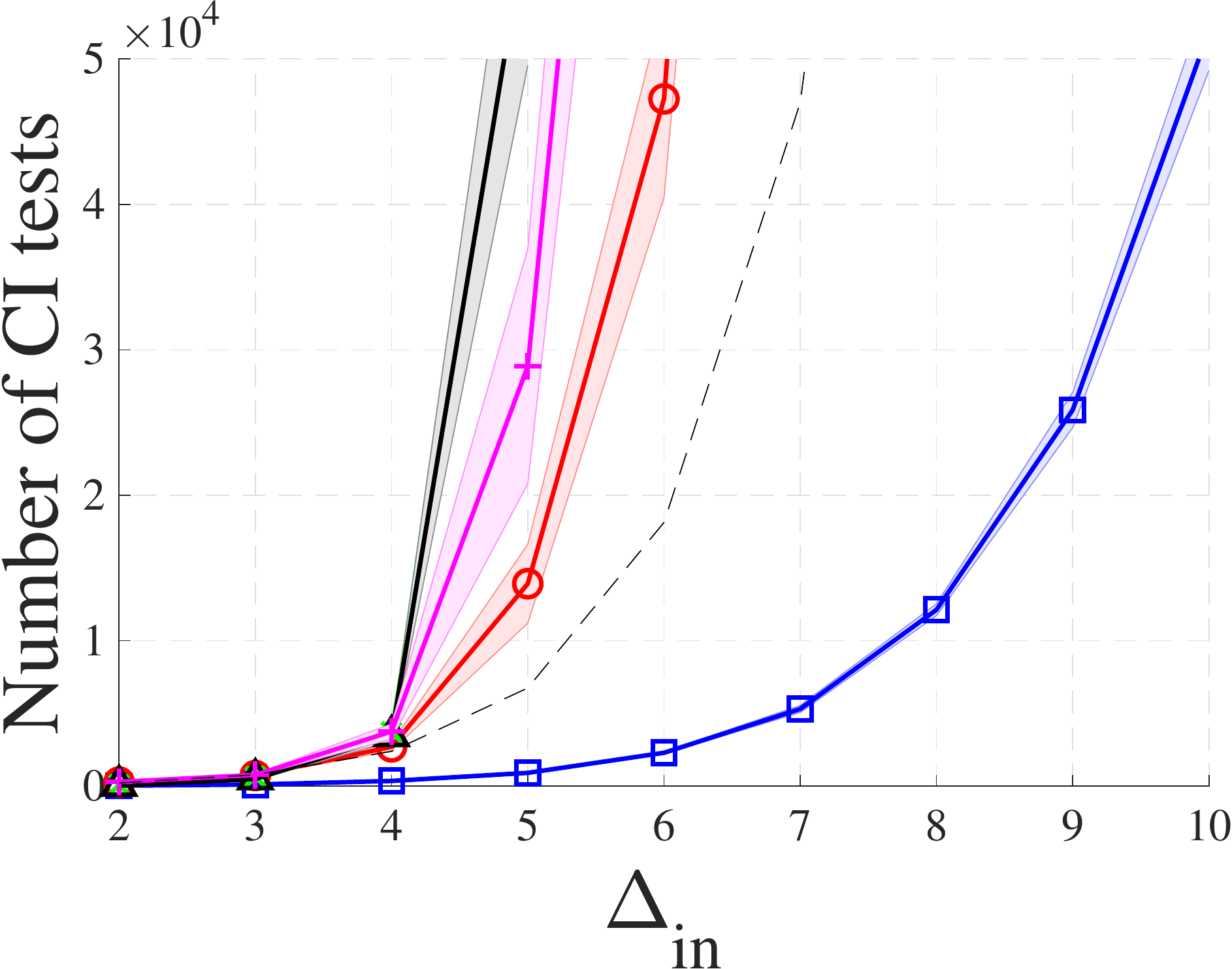}
                \hspace*{0.4cm}\includegraphics[width= 0.4 \textwidth]{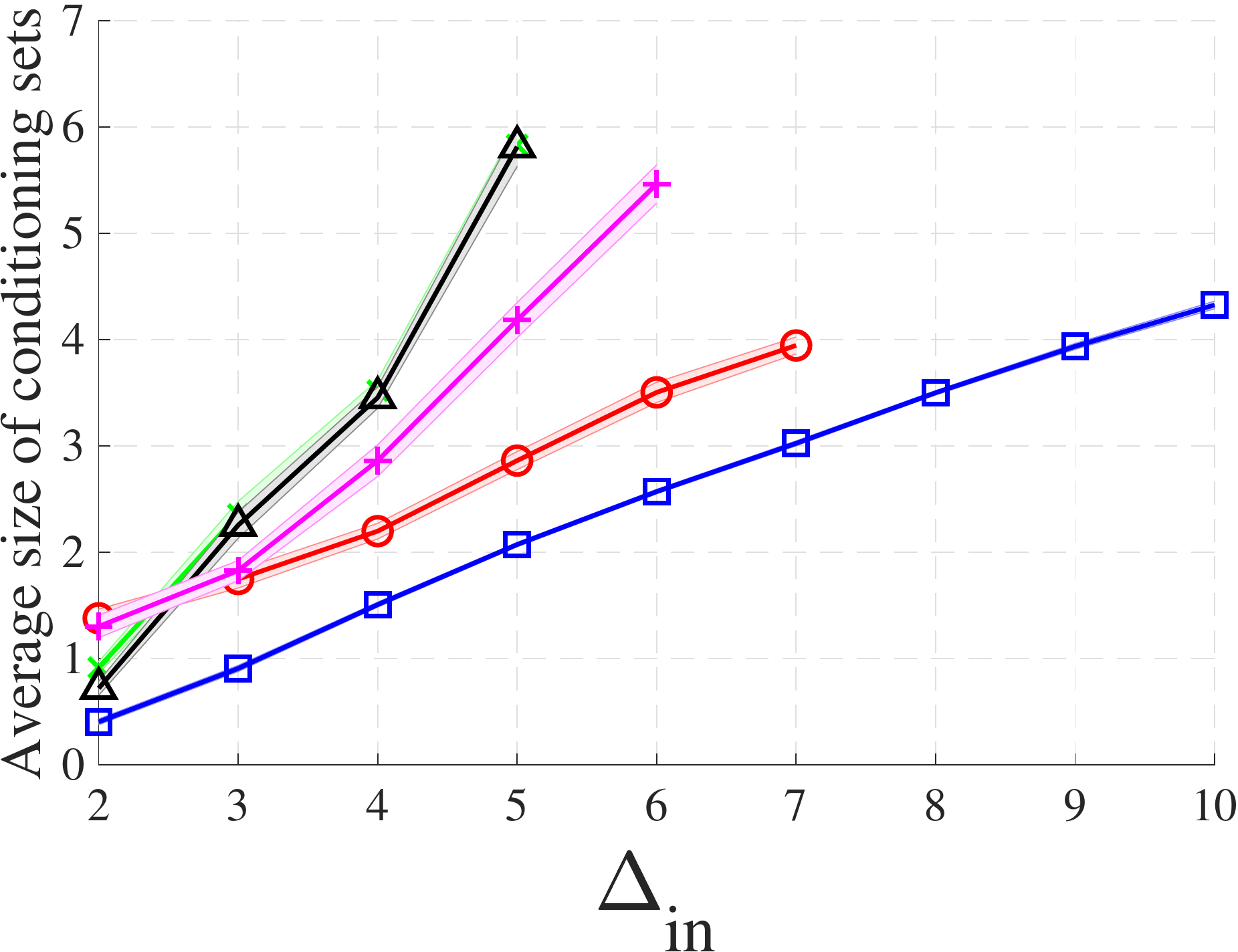}
                \caption{Fixed number of vertices ($p=25$)}
                \label{fig3, a:fixed order, varying indegree}
            \end{subfigure}
            
            \begin{subfigure}[b]{0.8\textwidth}
                \centering
                \hspace*{-0.4cm}\includegraphics[width = 0.4 \textwidth]{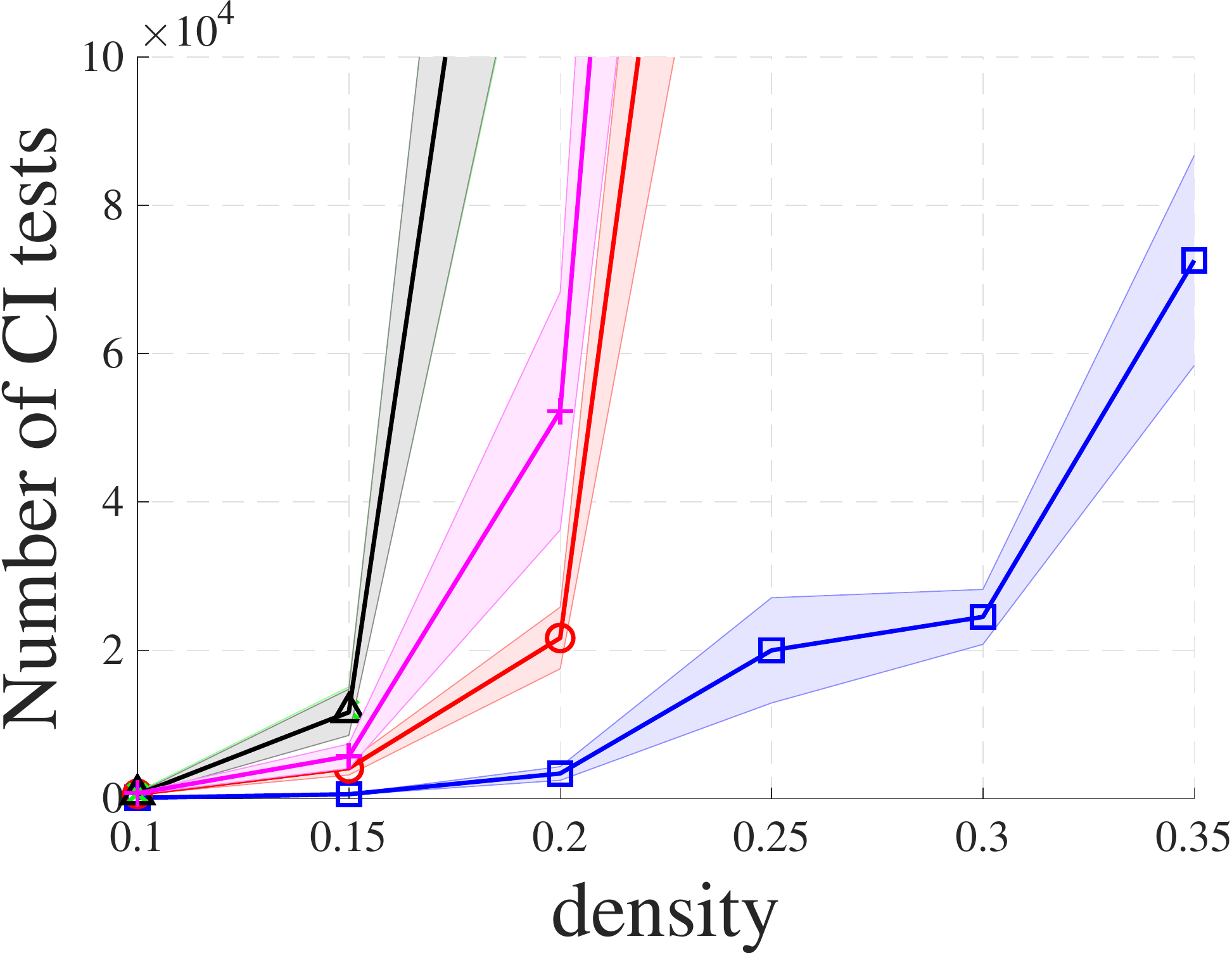}
                \hspace*{0.4cm}\includegraphics[width = 0.4 \textwidth]{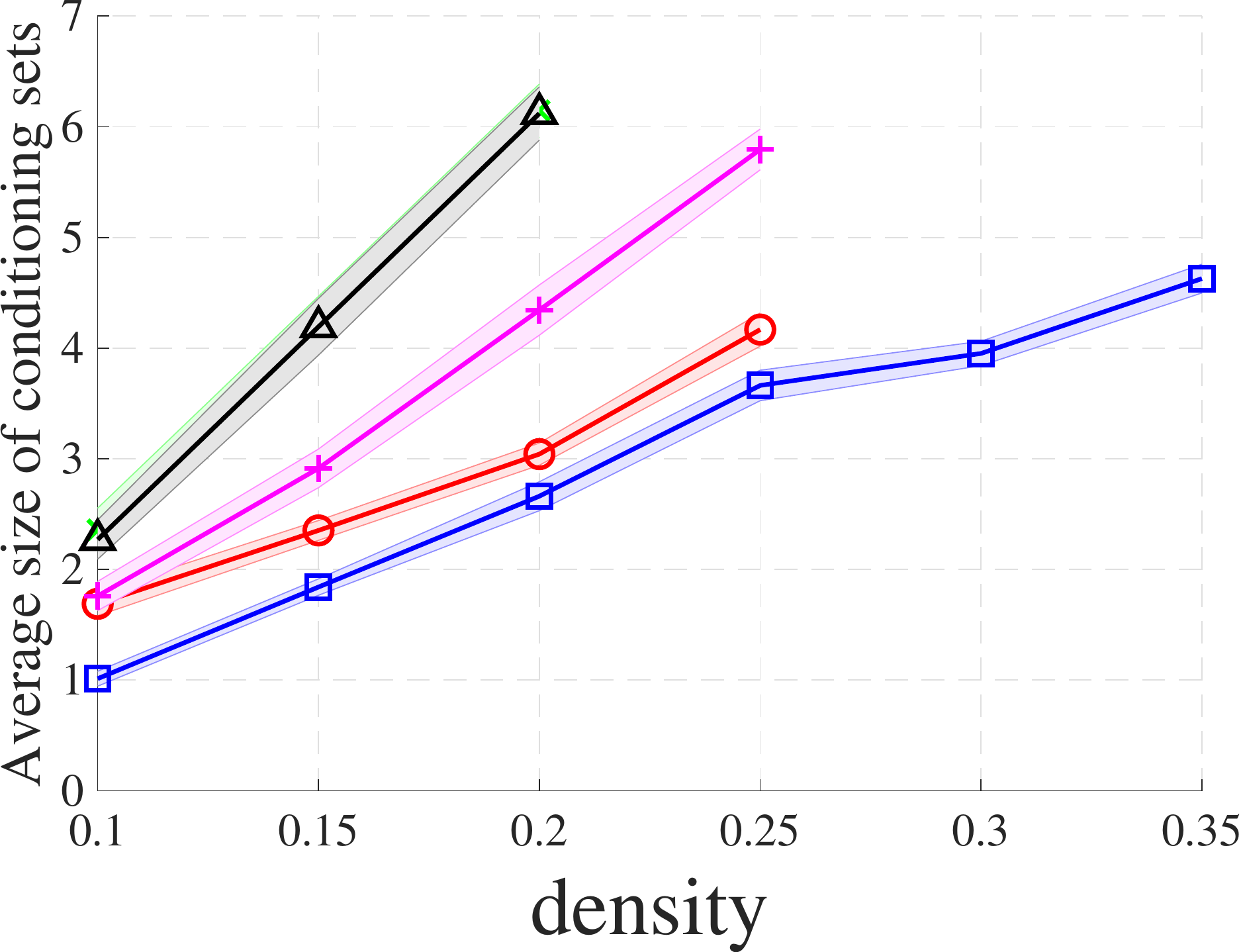}
                \caption{Fixed number of vertices ($p=25$)}
                \label{fig3, b:fixed order, varying density}
            \end{subfigure} 

            \begin{subfigure}[b]{0.8\textwidth}
                \centering
                \hspace*{-0.4cm}\includegraphics[width = 0.4 \textwidth]{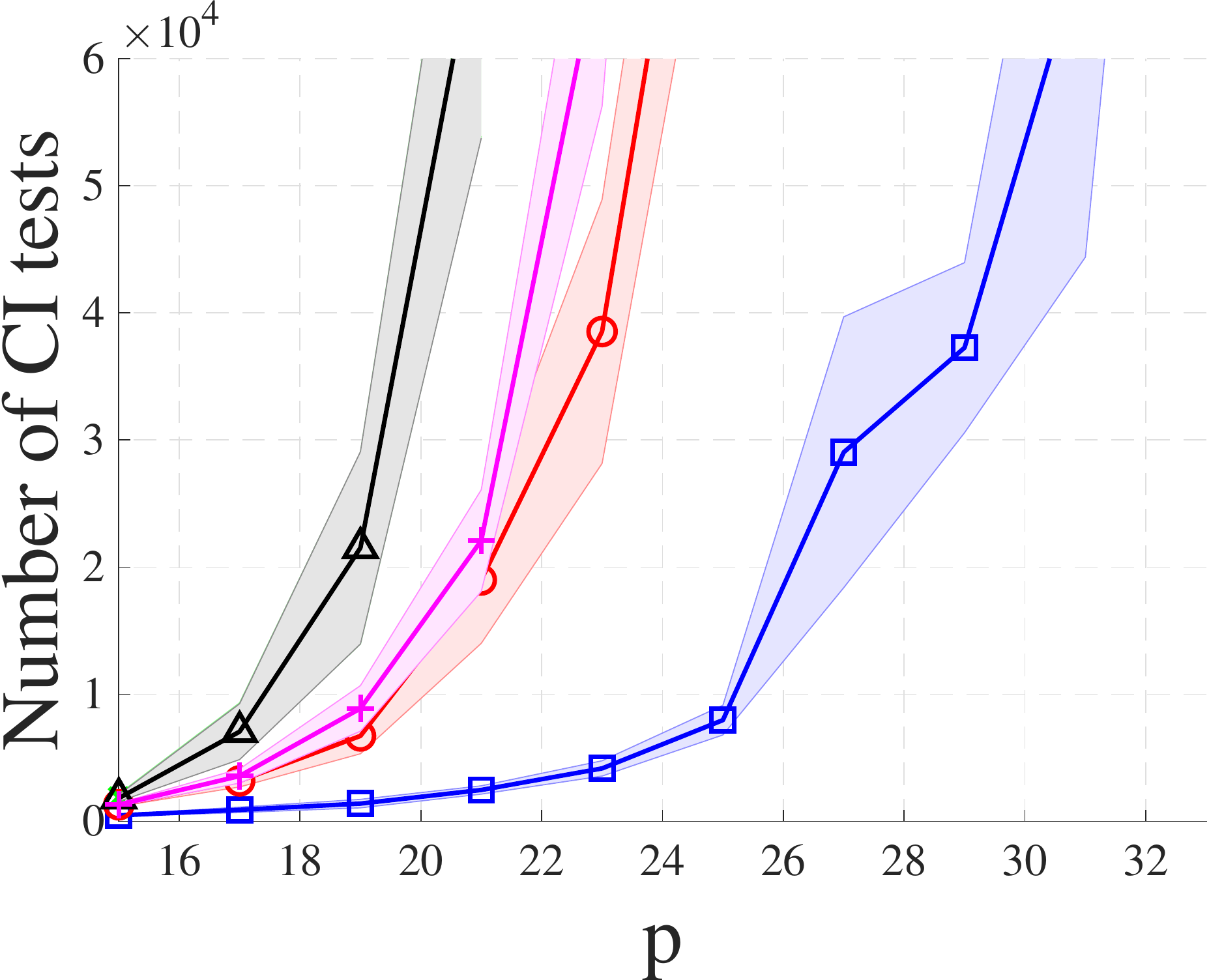}
                \hspace*{0.4cm}\includegraphics[width = 0.4 \textwidth]{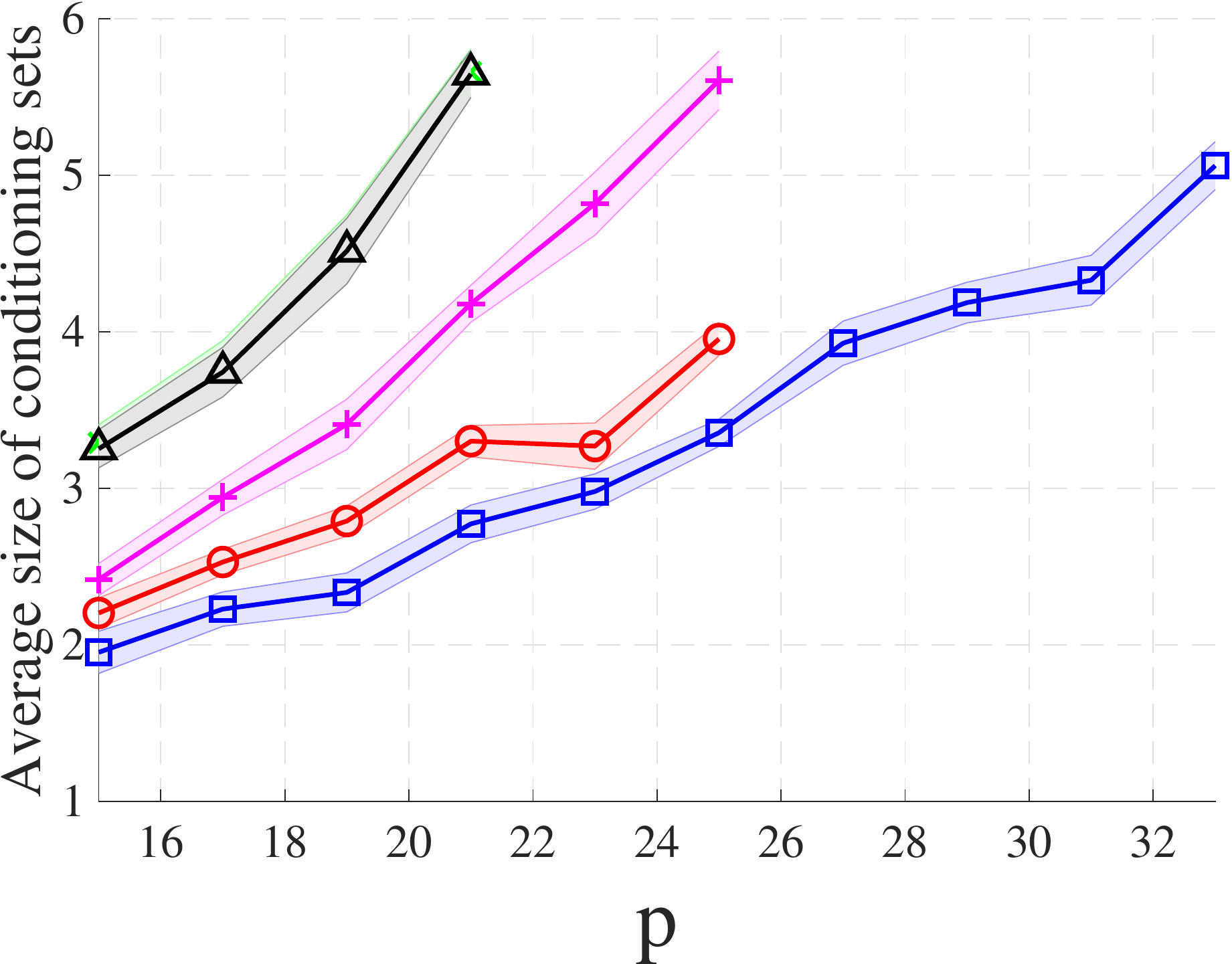}
                \caption{Fixed density  (density is 0.25)}
                \label{fig3, c:fixed density, varying order}
            \end{subfigure} 
            
            \begin{subfigure}[b]{0.8\textwidth}
                \centering
                \hspace*{-0.4cm}\includegraphics[width = 0.4 \textwidth]{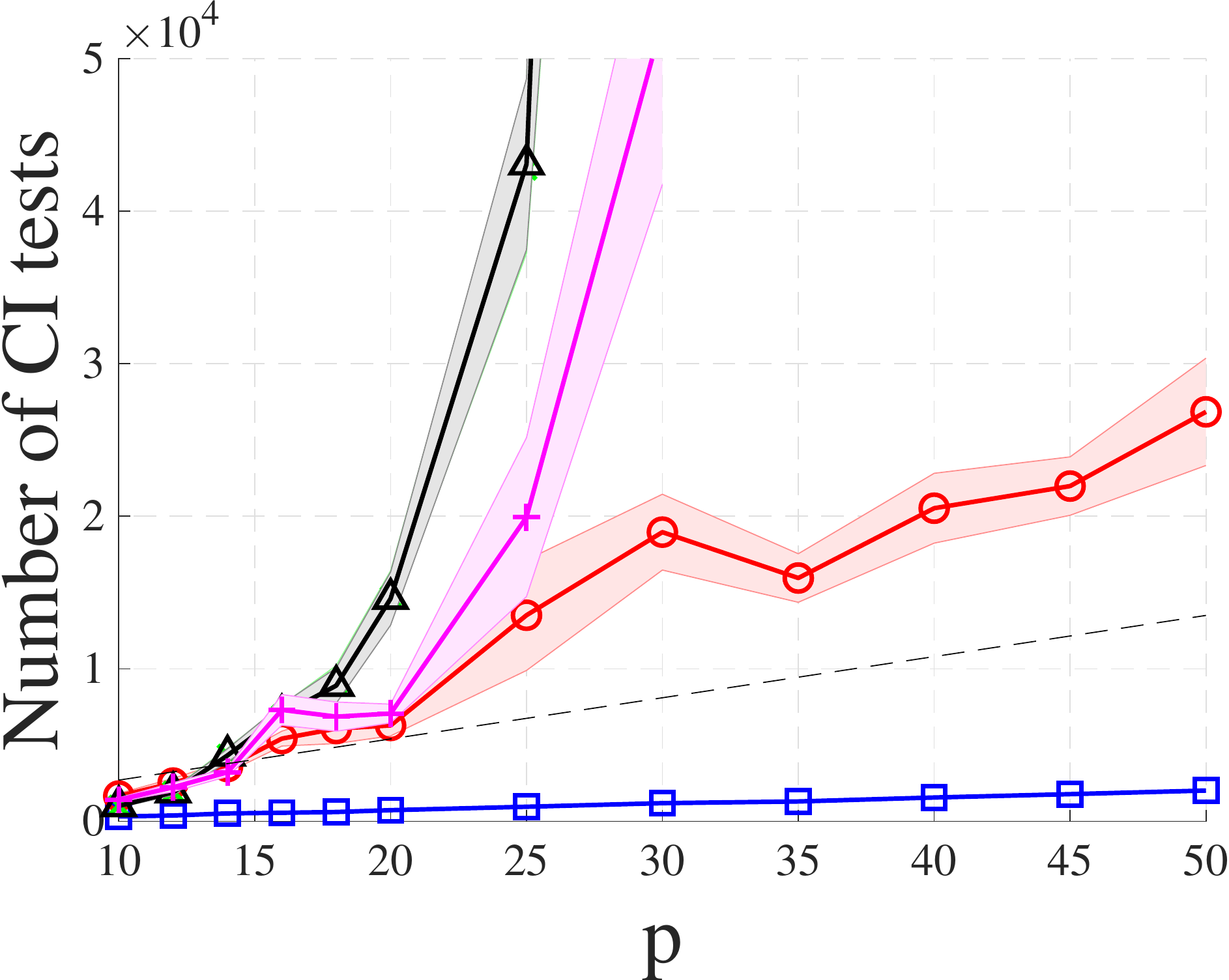}
                \hspace*{0.4cm}\includegraphics[width = 0.4 \textwidth]{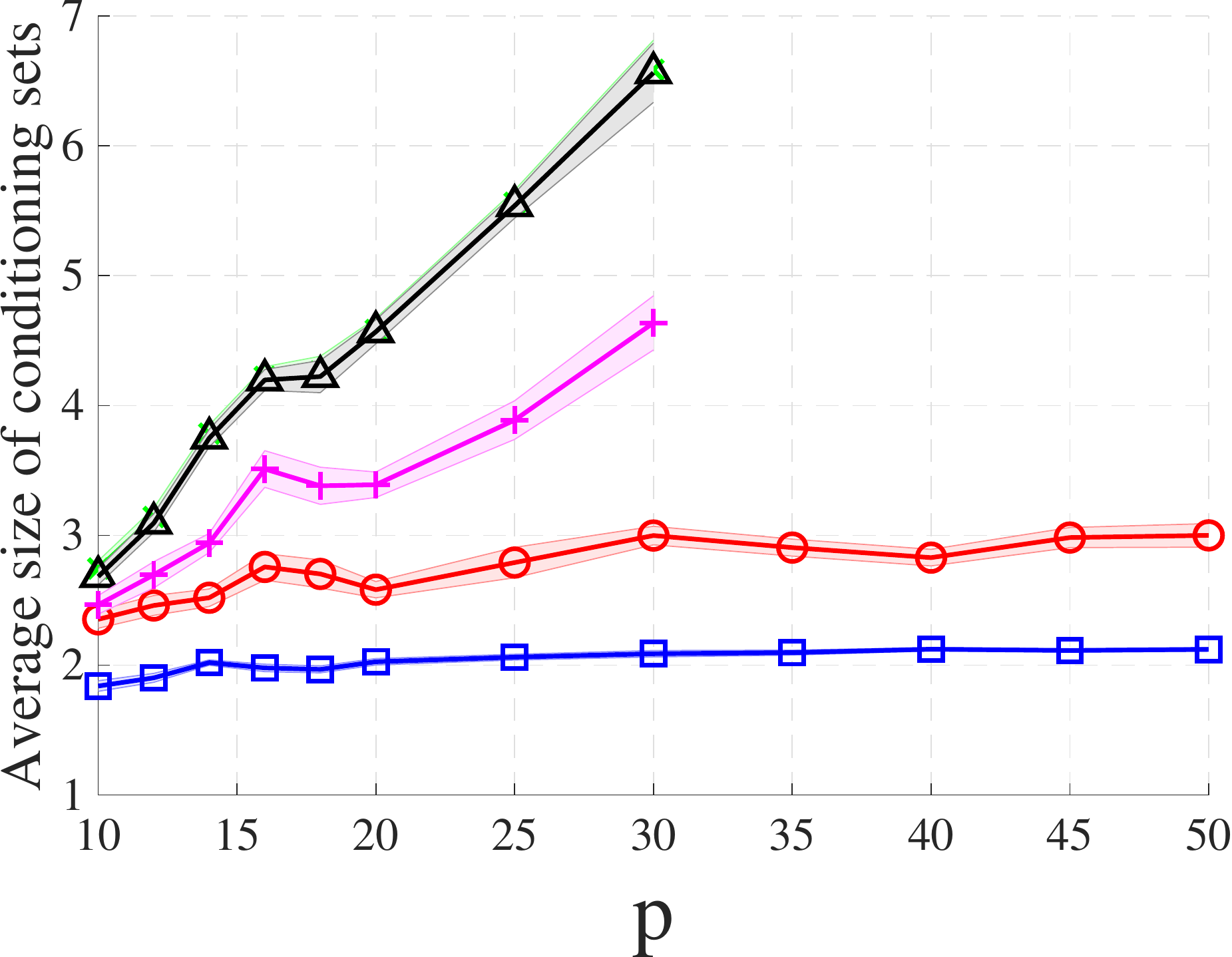}
                \caption{Fixed maximum in-degree ($\Delta_\text{in}=5$)}
                \label{fig3, d:fixed indegree, varying order}
            \end{subfigure} 
            \caption{Structure learning using oracle CI tests after Markov boundary discovery.}
            \label{fig: oracle}
        \end{figure*}
        
        We use d-separation relations in the causal DAG as the oracle answers for the CI tests in this setting.
        This is equivalent to having access to the joint distribution of the variables instead of a finite number of samples. 
        We report the number of performed CI tests and the size of conditioning sets which are the main factors determining the time complexity of the algorithms. 
        Note that CI tests with smaller conditioning sets are more reliable.
        
        We use two random graph models for generating the causal DAGs. 
        The first model is the directed Erd\H{o}s-R\`enyi model $G(p,m)$ \citep{erdHos1960evolution}, which provides uniform distribution over DAGs with $p$ vertices and $m$ edges (The skeleton is first sampled from the undirected Erd\H{o}s-R\`enyi model, and then the edges are oriented with respect to a random ordering over the variables). 
        As the second model, we use random DAGs with a fixed maximum in-degree $\Delta_{in}$, which are generated as follows. 
        We fix a random ordering of the variables. 
        For each vertex in the graph, we choose $\Delta_{in}$ potential parents among the other variables uniformly at random. 
        We then choose the parents that do not violate the ordering. 
        This procedure yields a DAG. 

        Figure \ref{fig: oracle} shows the experimental results regarding the number of CI tests. 
        Subfigures \ref{fig3, a:fixed order, varying indegree} and \ref{fig3, d:fixed indegree, varying order} depict the result for the fixed $\Delta_\text{in}$ model. 
        Subfigures \ref{fig3, b:fixed order, varying density} and \ref{fig3, c:fixed density, varying order} depict the result for the Erd\H{o}s-R\`enyi model, where the density of a graph is defined as the number of edges divided by the maximum possible number of edges, $\binom{p}{2}$. 
        Each point in the plots is obtained using 20 DAGs. 
        The shaded bars denote the $80\%$ confidence intervals \citep{shadederror}. 
        As seen in the figures, MARVEL requires substantially fewer CI tests with smaller conditioning sets, and outperforms all other algorithms, especially on the graphs with fixed max in-degree. 
        Also, as witnessed by Subfigures \ref{fig3, a:fixed order, varying indegree} and \ref{fig3, d:fixed indegree, varying order}, even the worst-case upper bound of our proposed method is below the number of CI tests required by the other algorithms.
        
    \subsection{Finite Sample Setting} \label{sec: noisy}
        In this setting, we have access to a finite number of samples from $P_\mathbf{V}$. 
        We evaluate the algorithms on various real-world structures available at Bayes Network Repository \citep{bnrepository} along with random graphs described in the last part.
        
        We compare the algorithms in three scenarios. 
        \subsubsection{Scenario 1: The effect of sample size on real-world structures}
            \begin{figure*}[t]
            \centering
            \begin{subfigure}[b]{\textwidth}
                \centering
                \includegraphics[width=0.5\textwidth]{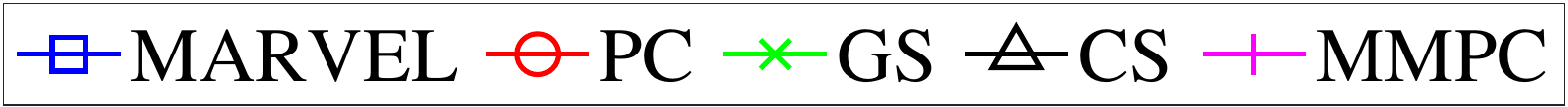}
            \end{subfigure}
            
            \begin{subfigure}[b]{1\textwidth}
                \centering
                \includegraphics[width=0.3\textwidth]{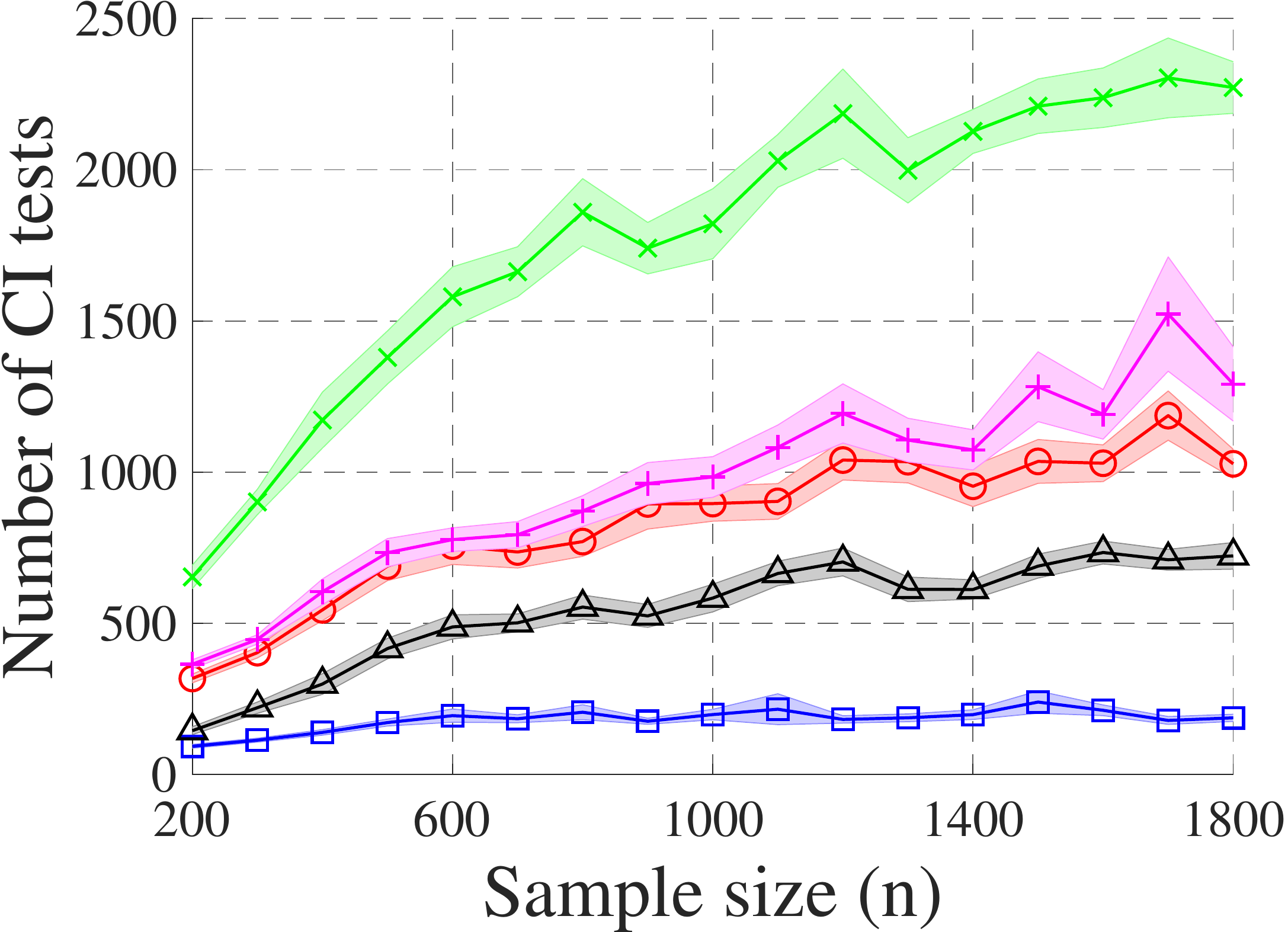}
                \hfill
                \includegraphics[width=0.3\textwidth]{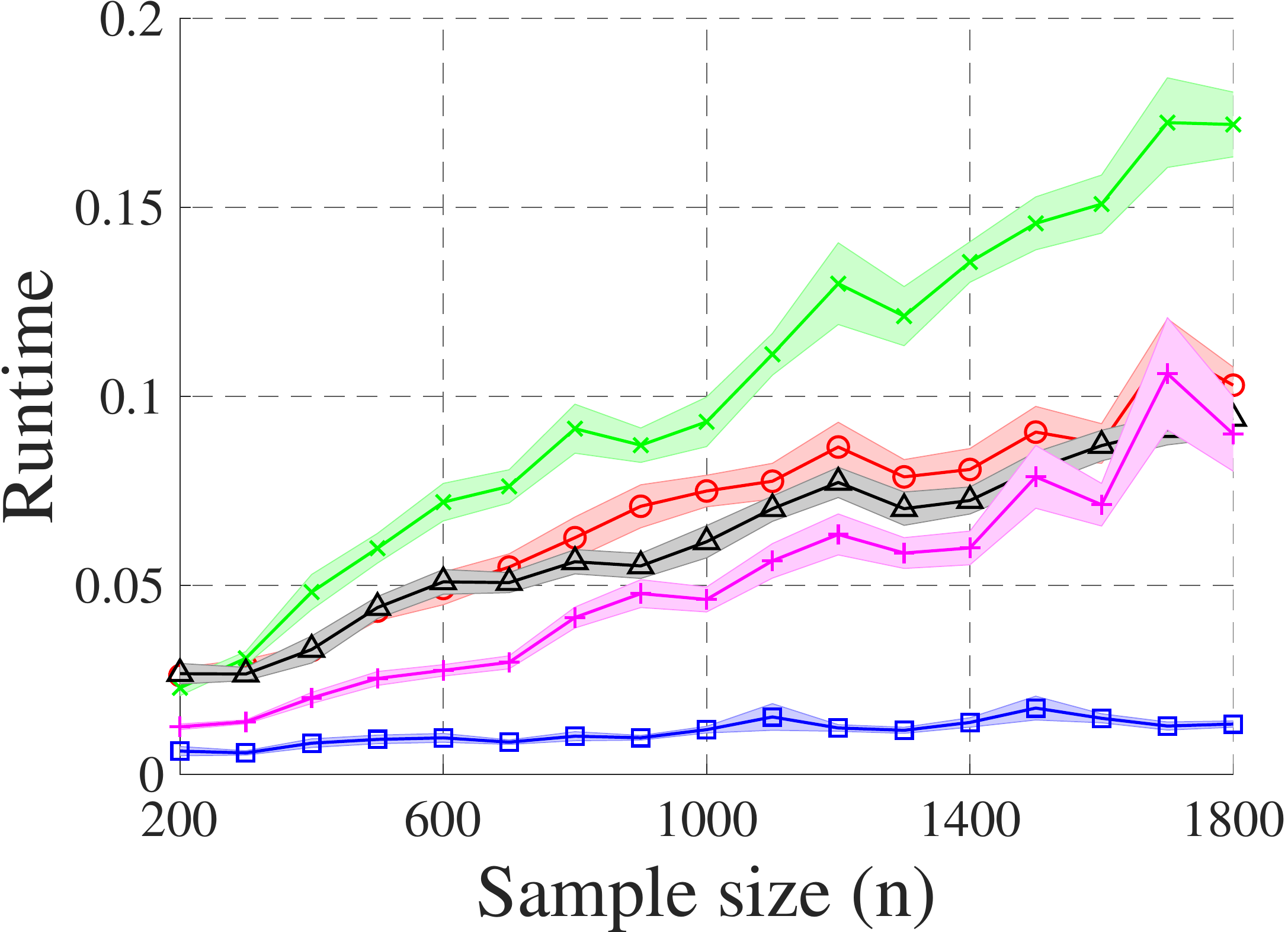}
                \hfill
                \includegraphics[width=0.3\textwidth]{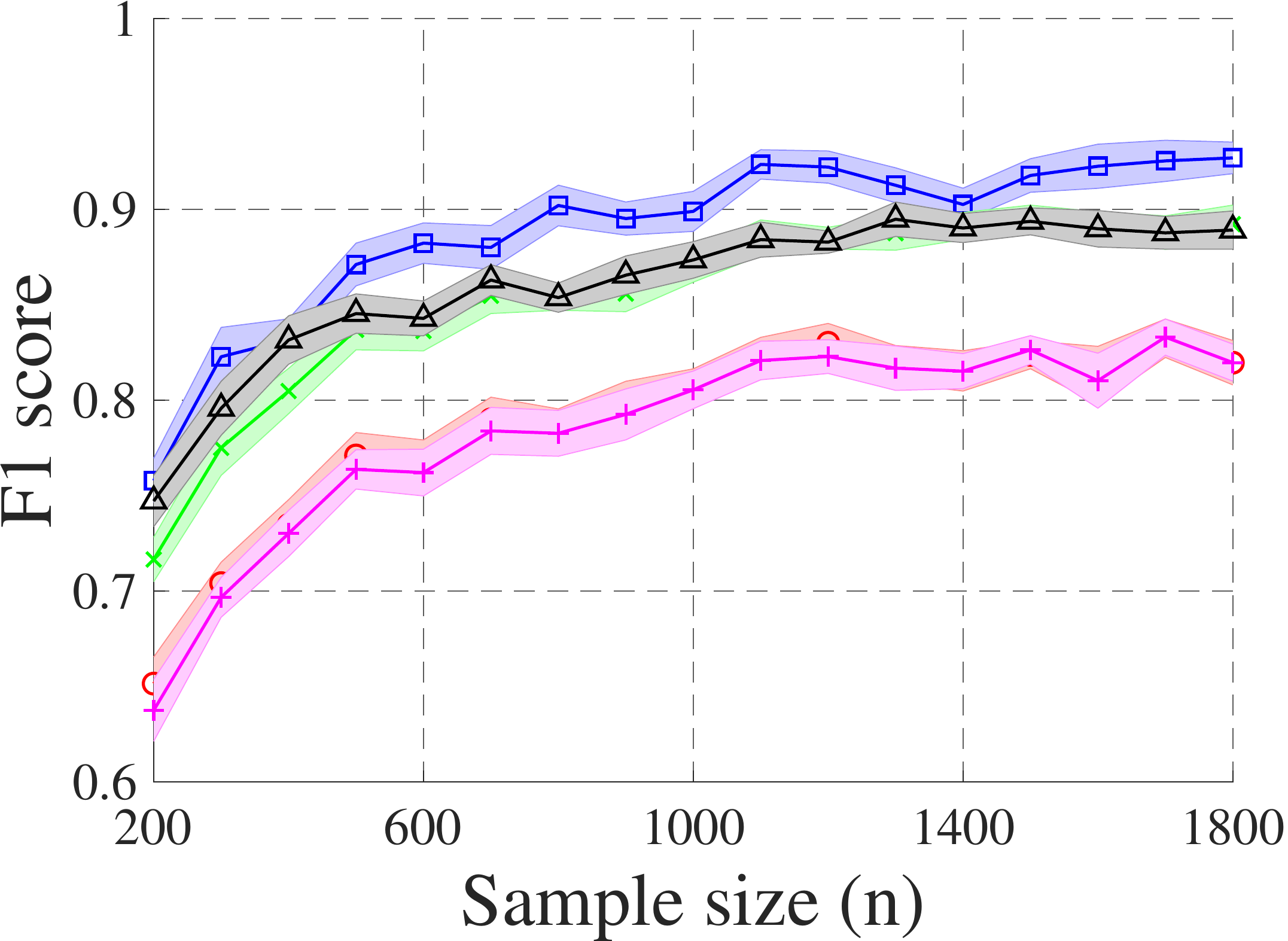}
                \caption{Insurance ($\left\vert V\right\vert=27, \left\vert E\right\vert=51$)}
            \end{subfigure}
            
            \begin{subfigure}[b]{1\textwidth}
                \centering
                \includegraphics[width=0.3\textwidth]{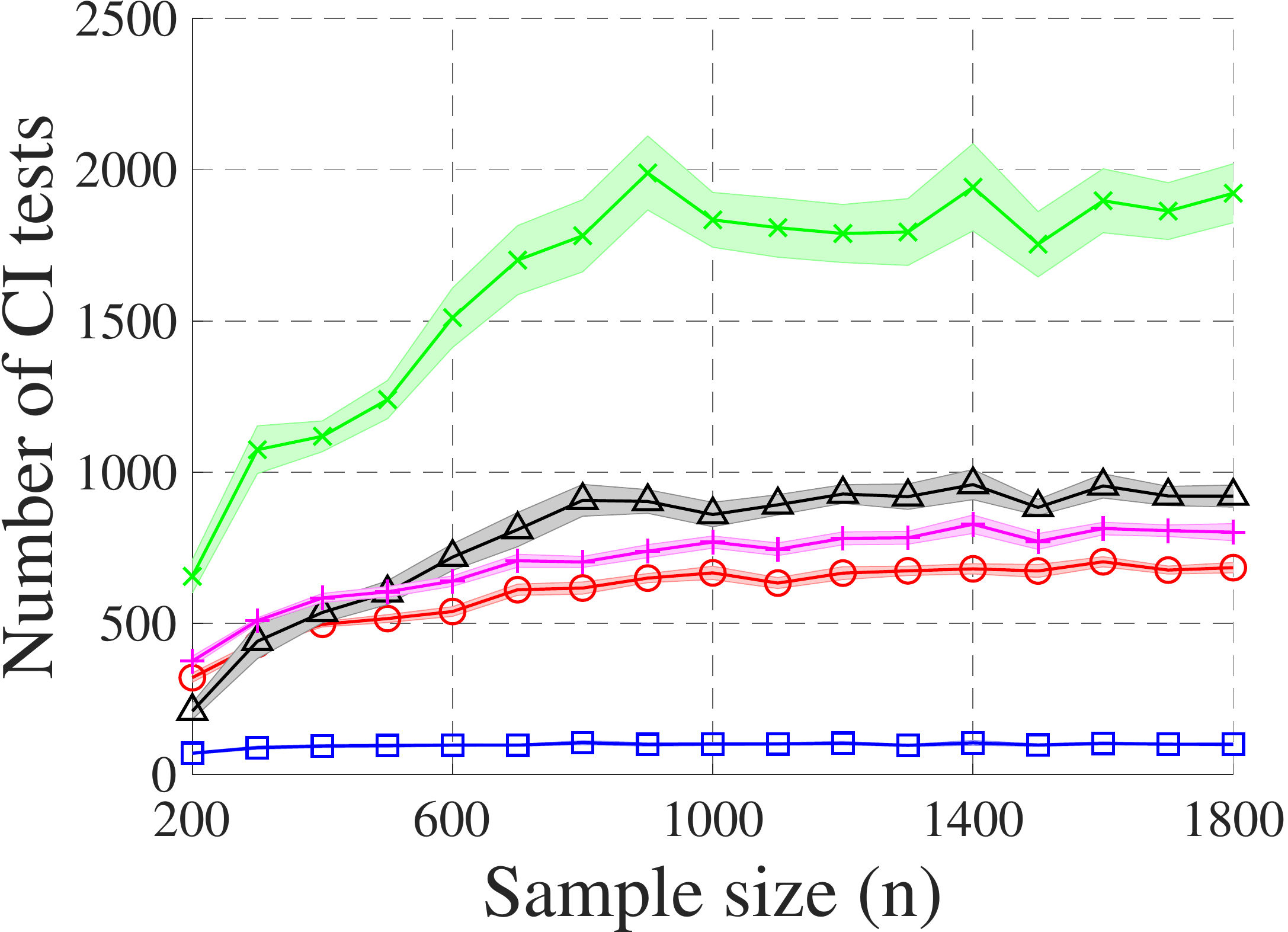}
                \hfill
                \includegraphics[width=0.3\textwidth]{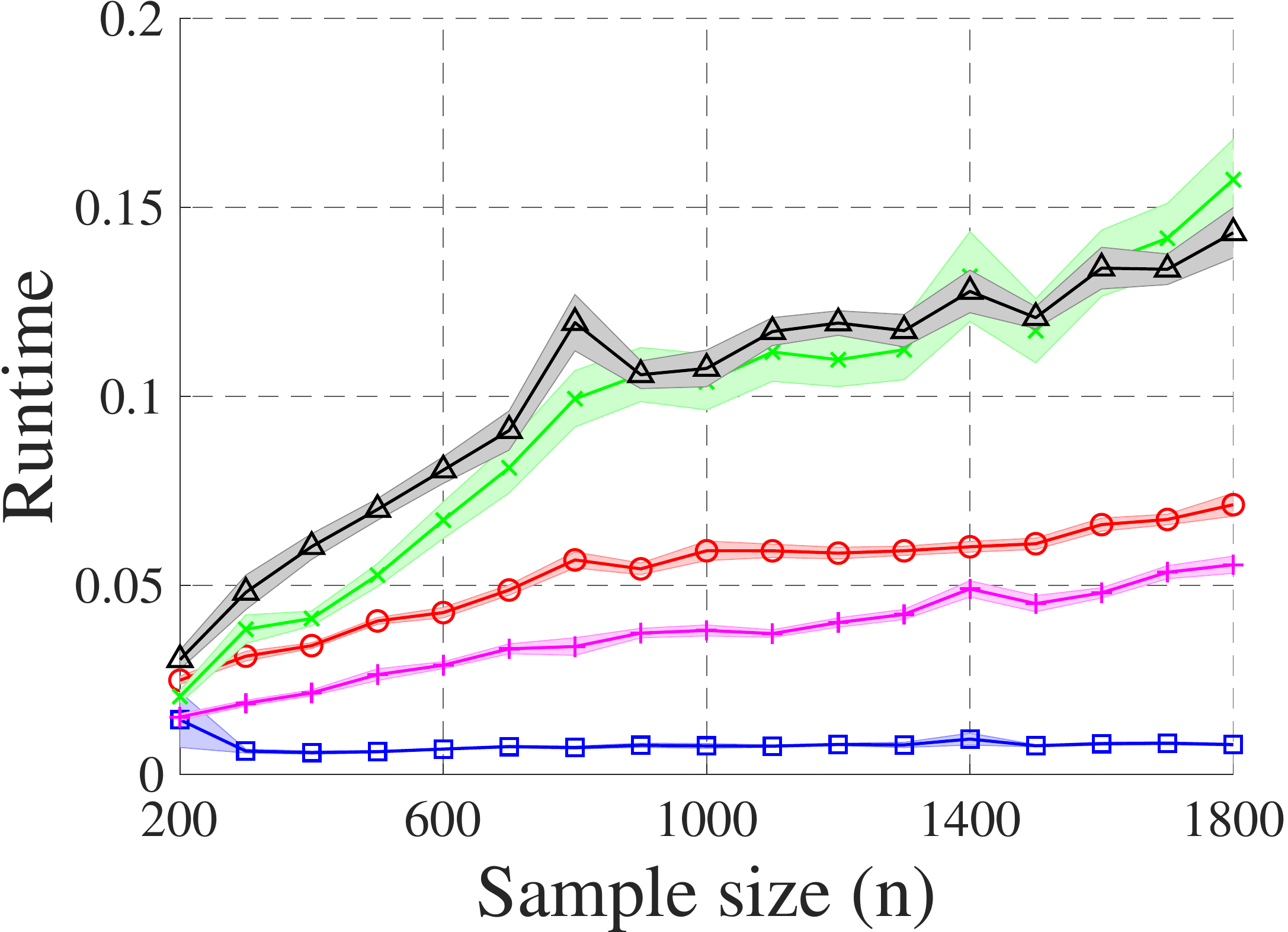}
                \hfill
                \includegraphics[width=0.3\textwidth]{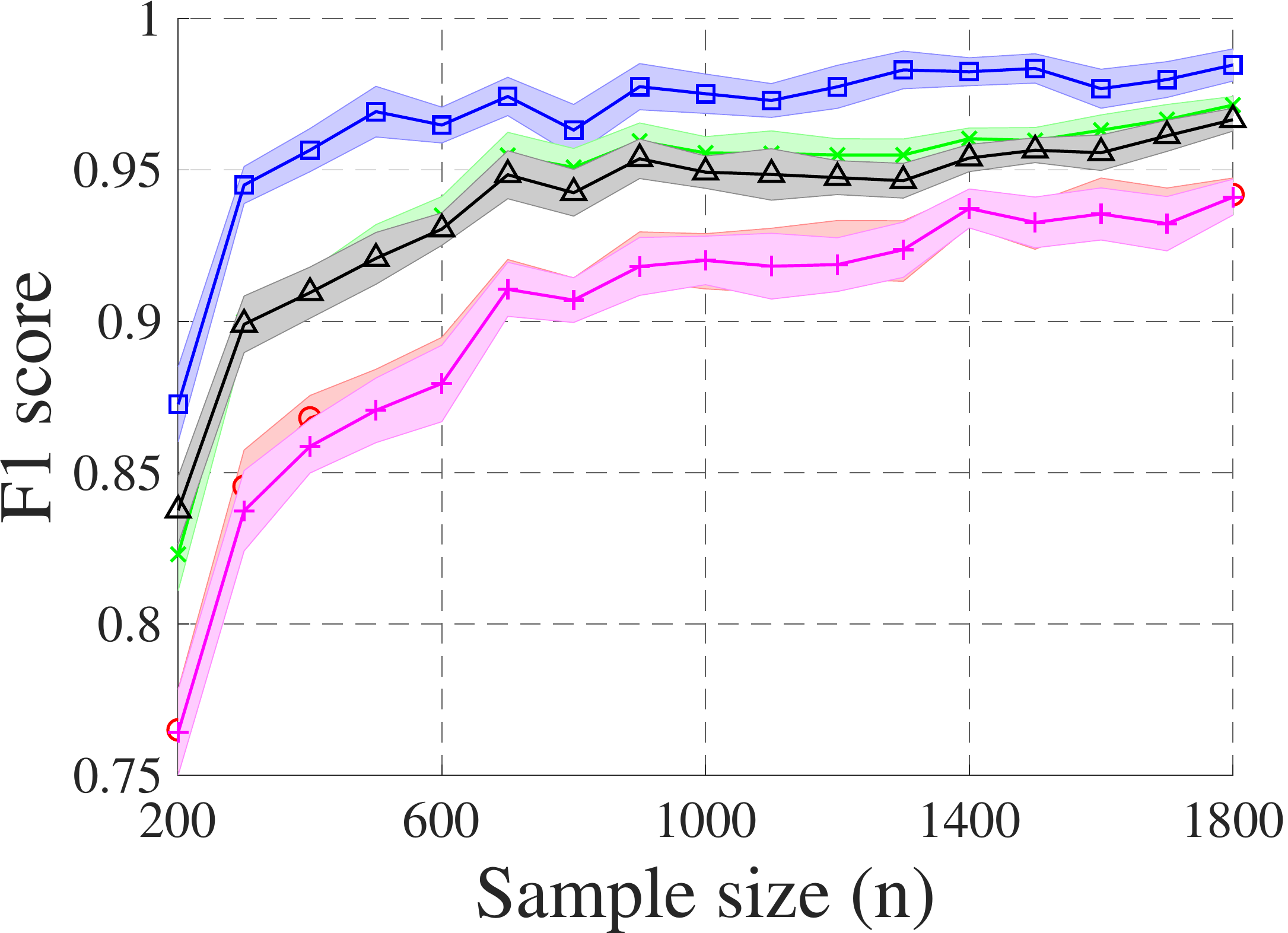}
                \caption{Mildew ($\left\vert V\right\vert=35, \left\vert E\right\vert=46$)}
            \end{subfigure}
            
            \begin{subfigure}[b]{1\textwidth}
                \centering
                \includegraphics[width=0.3\textwidth]{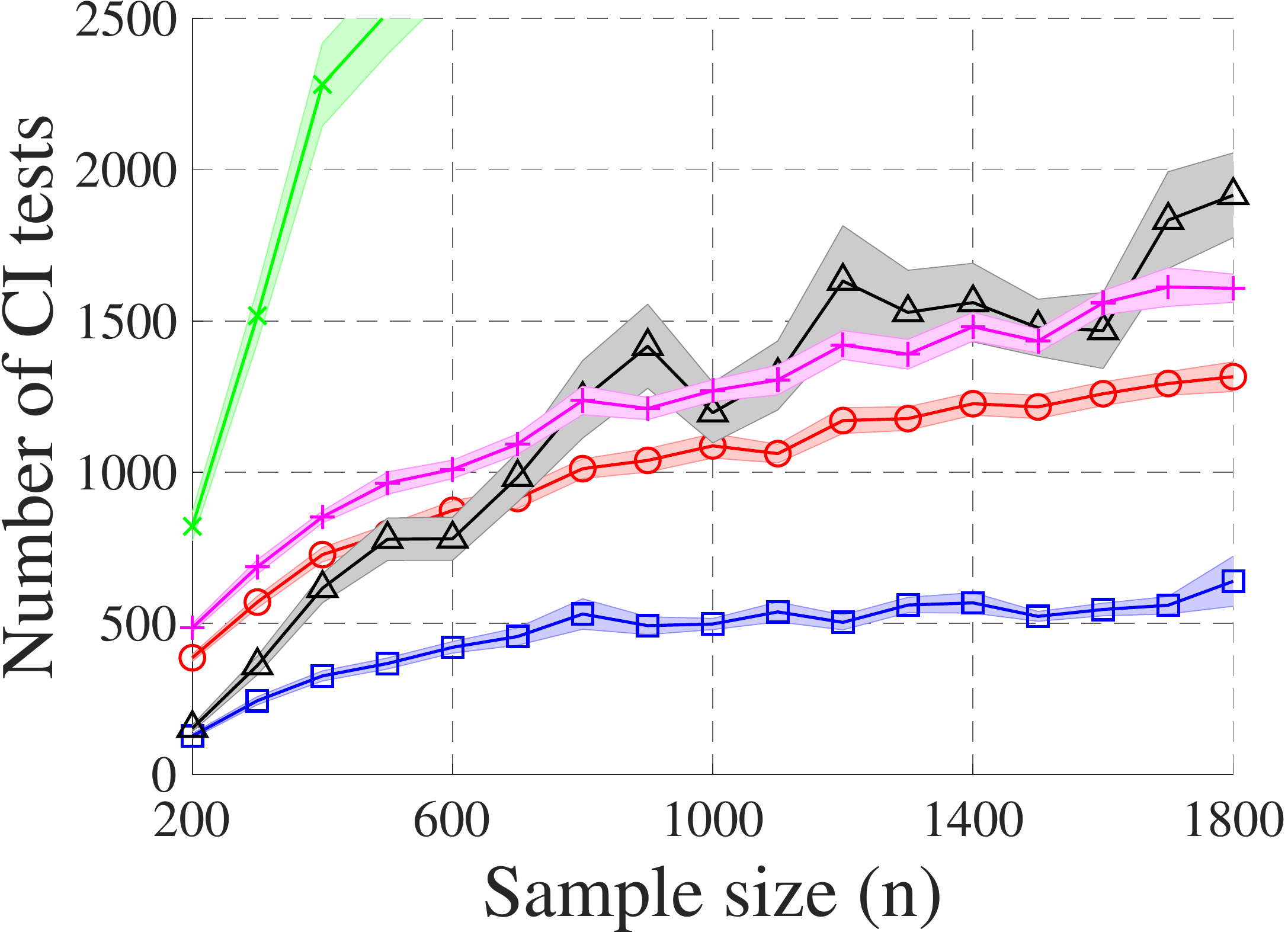}
                \hfill
                \includegraphics[width=0.3\textwidth]{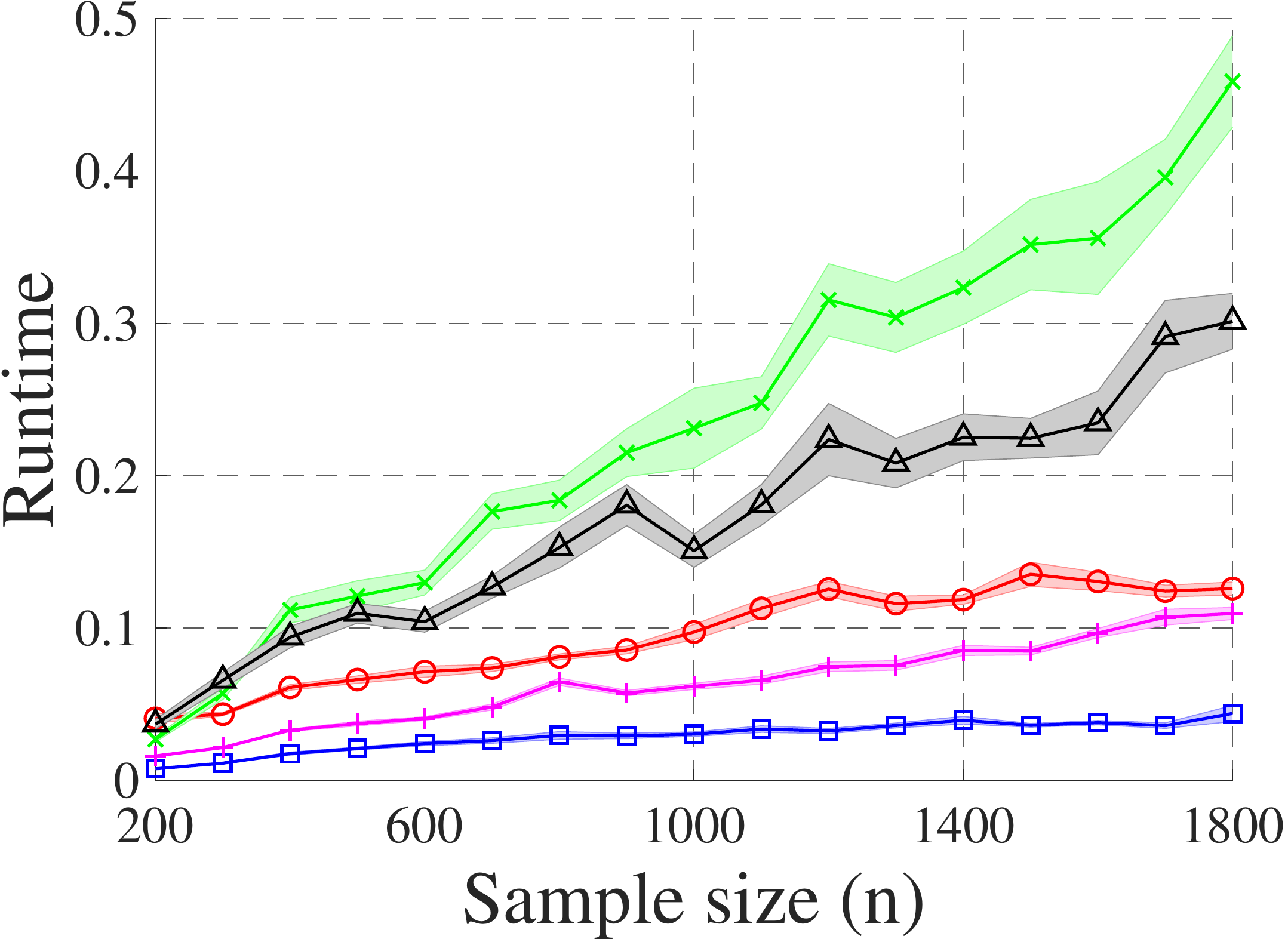}
                \hfill
                \includegraphics[width=0.3\textwidth]{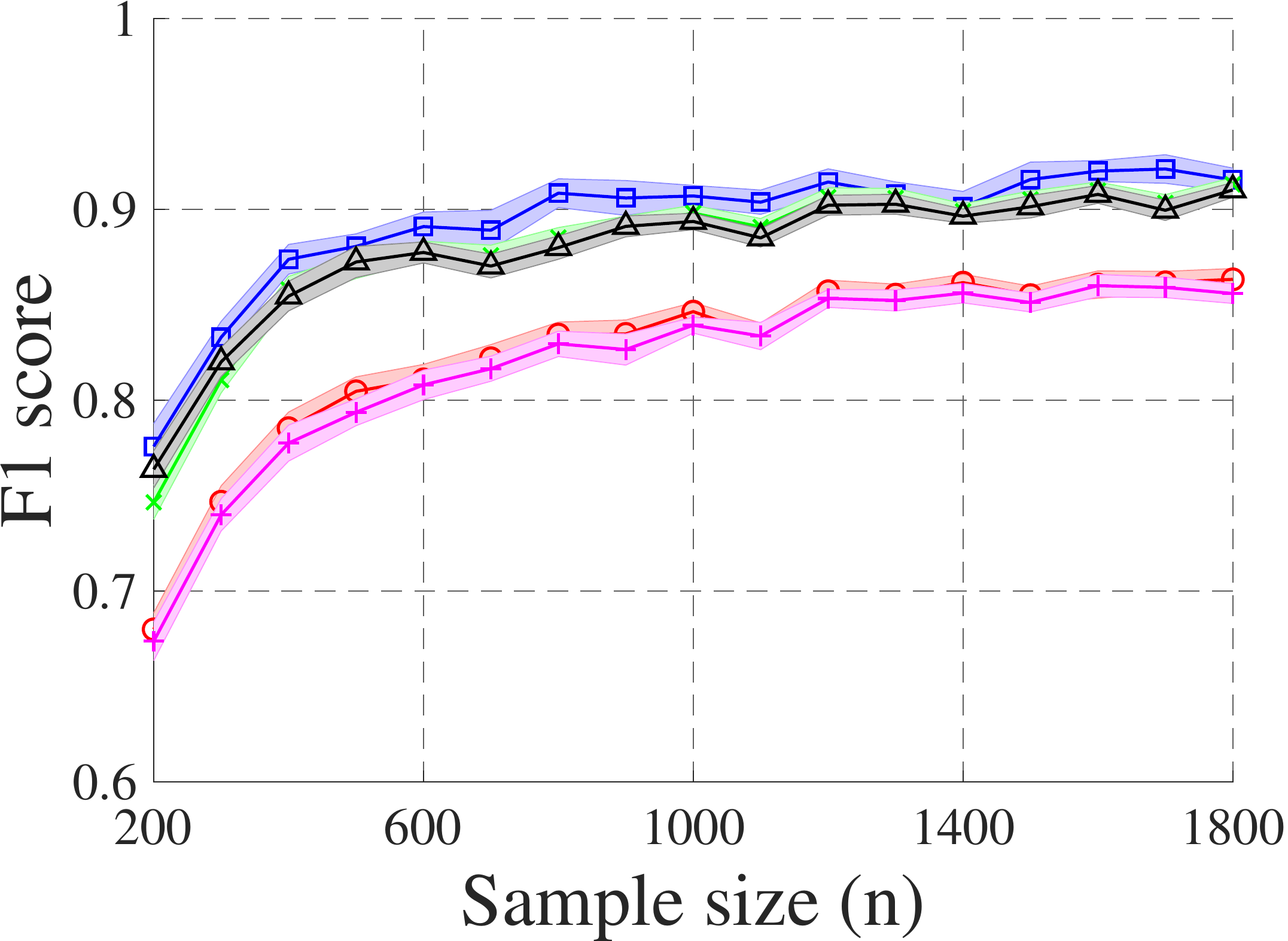}
                \caption{Barley ($\left\vert V\right\vert=48, \left\vert E\right\vert=84$)}
            \end{subfigure}
            
            \begin{subfigure}[b]{1\textwidth}
                \centering
                \includegraphics[width=0.3\textwidth]{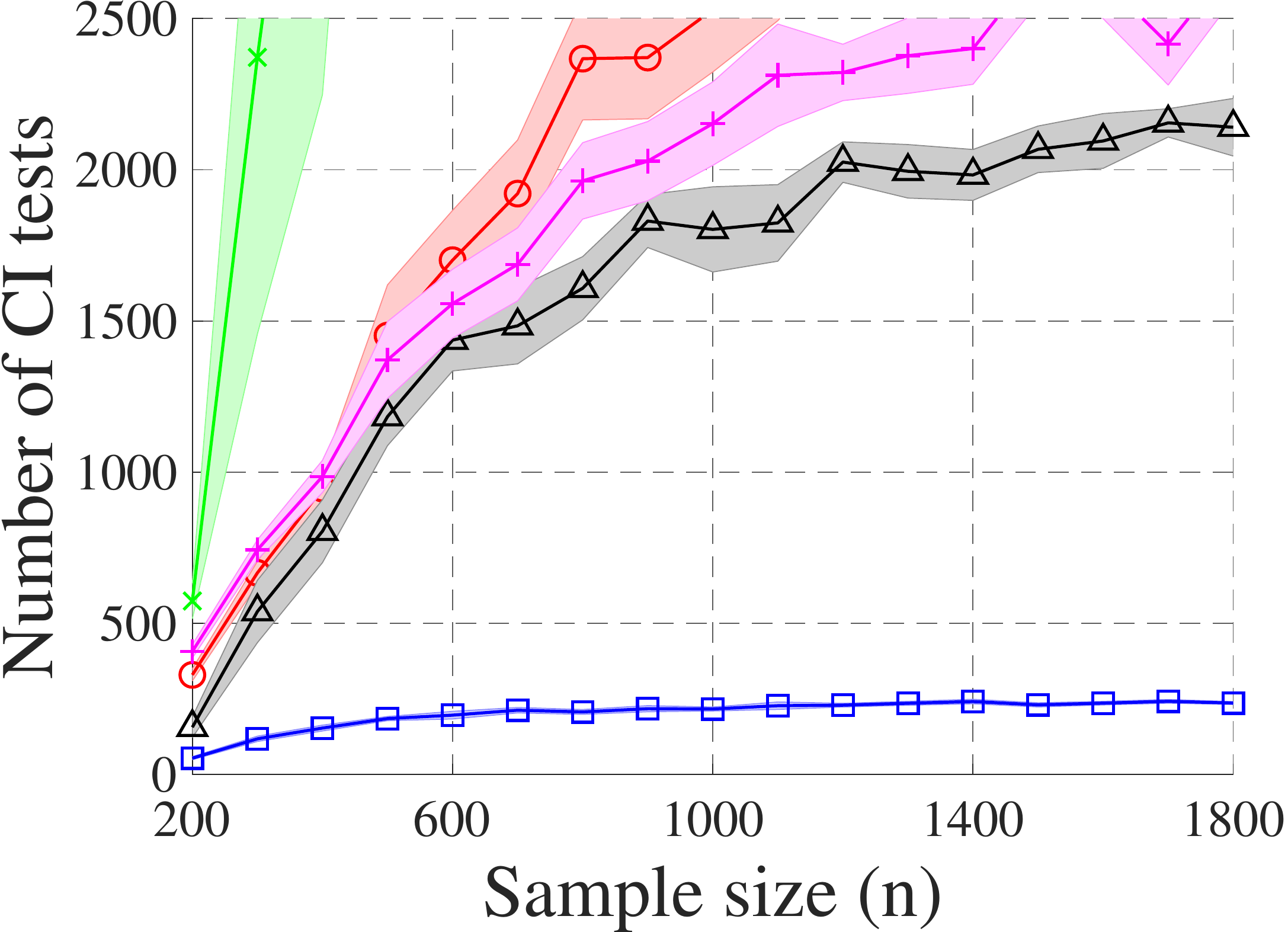}
                \hfill
                \includegraphics[width=0.3\textwidth]{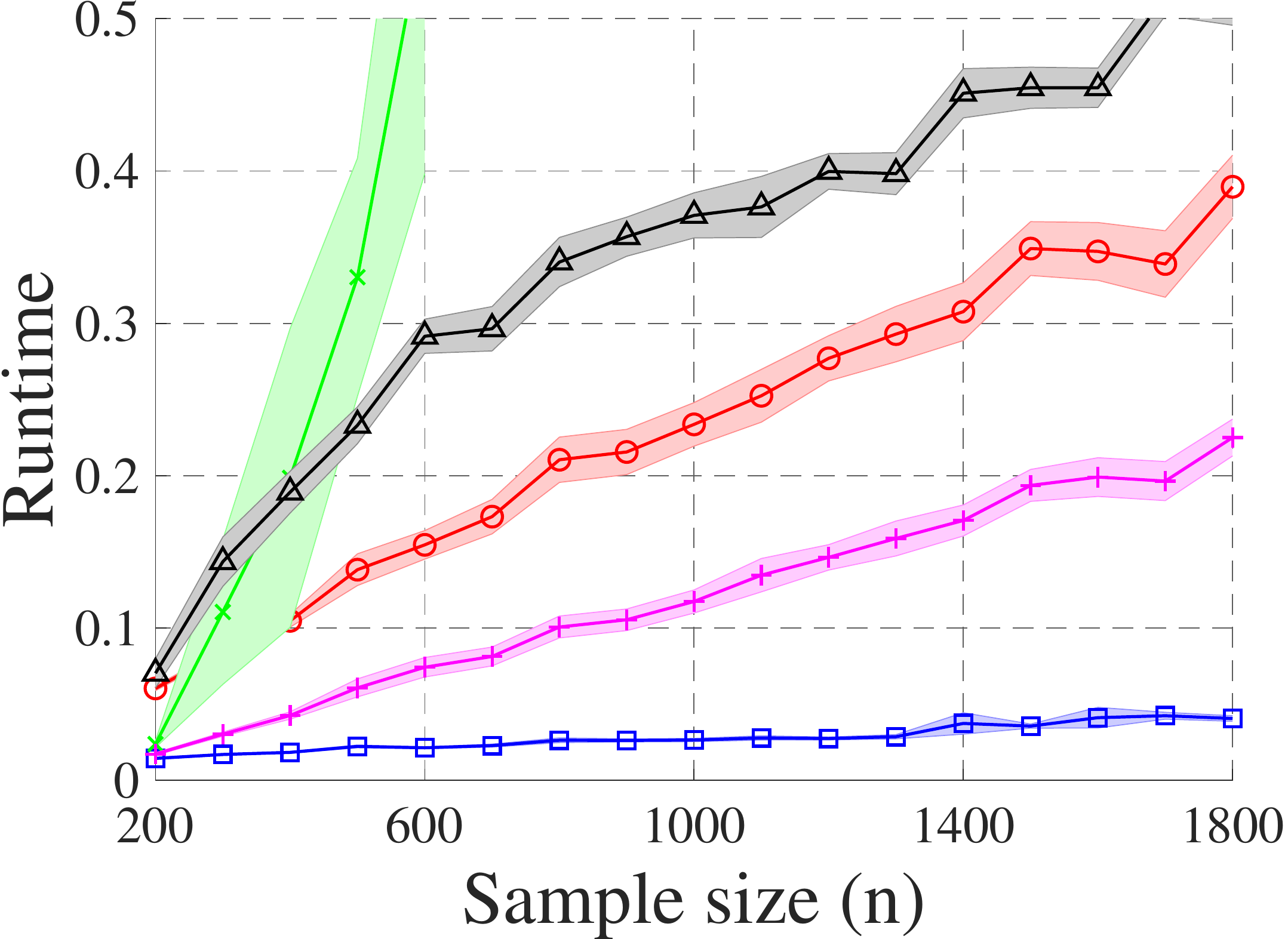}
                \hfill
                \includegraphics[width=0.3\textwidth]{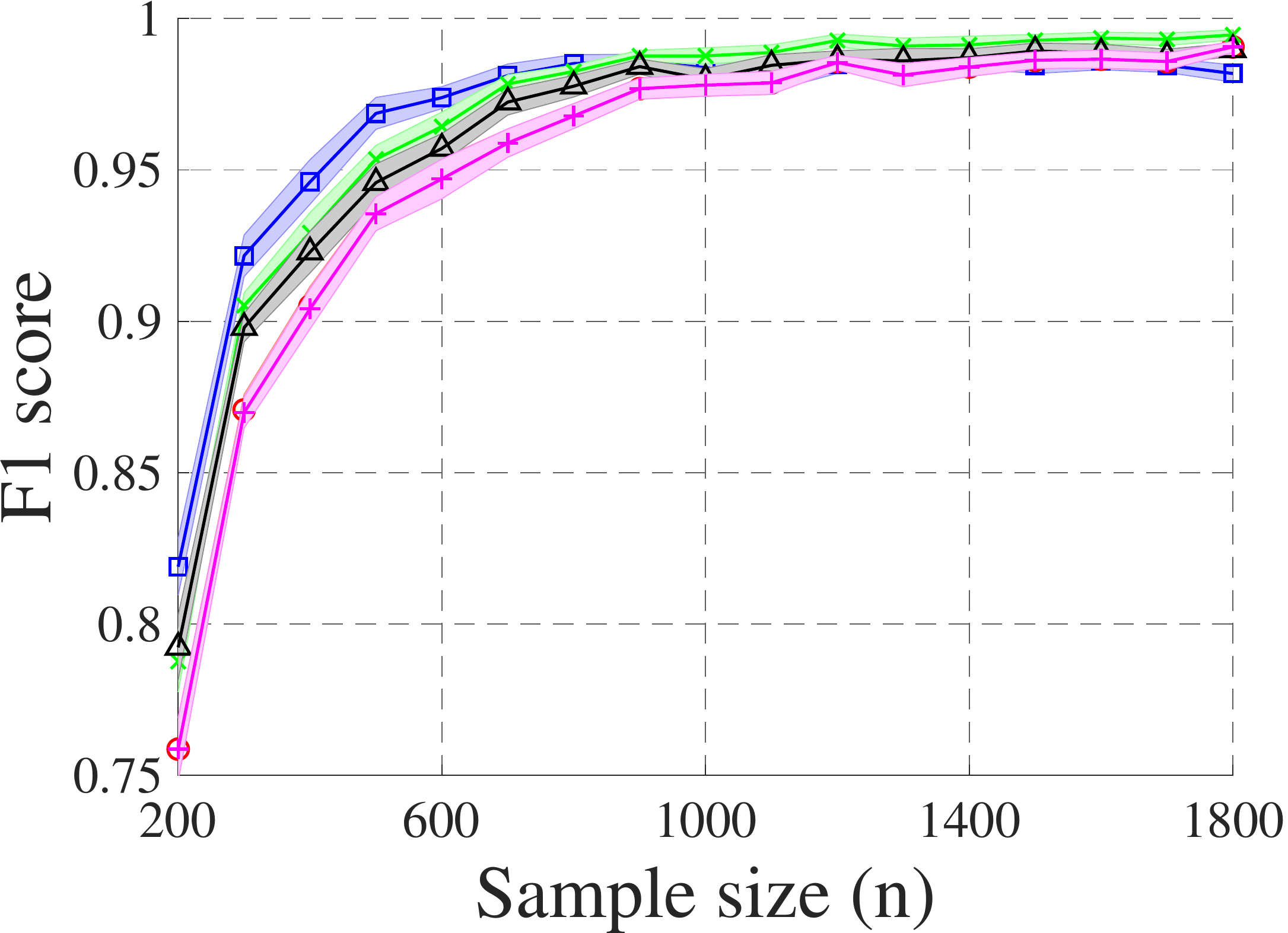}
                \caption{Win95pts ($\left\vert V\right\vert=76, \left\vert E\right\vert=70$)}
            \end{subfigure}
            \caption{Scenario 1: The effect of sample size on the performance of the structure learning algorithms for real-world structures after Markov boundary discovery.}
            \label{fig: noisy}
        \end{figure*}
            In this setting, we have picked 4 real-world structures, namely Insurance, Mildew, Barley, and Win95pts.
            After fixing the graph, the data is generated from a linear Gaussian structural causal model \citep{pearl2009causality}, where each variable is generated as a linear combination of its parents plus a Gaussian noise.
            The coefficients are chosen uniformly at random from the interval $[-1,-0.5] \cup [0.5,1]$ and the noise variables are distributed according to $\mathcal{N}(0,\sigma_X^2)$ where $\sigma_X$ is chosen uniformly at random from the interval $[1, \sqrt{3}]$. 
            The performance of the algorithms is measured by 
            \begin{enumerate*}
                \item the number of performed CI tests,
                \item runtime\footnote{Numbers are in seconds.}, and
                \item F1 score of the learned skeleton.
            \end{enumerate*}
            We have used TC algorithm for Markov boundary discovery for all of the algorithms.
            Moreover, we have used Fisher Z-transformation \citep{fisher1915frequency} to perform the CI tests with parameter $\alpha= \frac{2}{p^2}$, following the convention in \citep{pellet2008using}, which includes an analysis over the choice of this parameter.

        Figure \ref{fig: noisy} illustrates the results of this scenario.
        The reported results for the number of CI tests and runtime is after Markov boundary discovery.
        As seen in this figure, compared to the other algorithms, MARVEL is faster, requires a smaller number of CI tests, and obtains the highest accuracy in nearly all cases. 
        \subsubsection{Scenario 2: random graphs}
        Data generating process in this section is the same as the previous scenario, and the algorithms are evaluated on a set of larger graphs.
        
        \begin{table*}[h] 
            \caption{Scenario 2: Performance of the causal structure learning algorithms on random graphs after Markov boundary discovery ($n=50p$, $\Delta_\text{in}=4$).}
            \fontsize{9}{10.5}\selectfont
    	    \centering
    	    \begin{tabular}{N|M{1.3cm}|M{1.3cm}||M{1cm}|M{1cm}|M{1cm}|M{1cm}|M{1cm}|M{1cm}|M{1cm}|M{1cm}|M{1cm}|}
        		\hline
    			& \multicolumn{2}{c||}{$p$}
    			& 50
    			& 60
    			& 70
    			& 80
    			& 90
    			& 100
    			& 150
    			& 200\\
                \hline
    			\hline
    			& \multirow{6}{*}{MARVEL}
    			& CI tests
    			& \textbf{1,567}& \textbf{1,586}& \textbf{1,494}& \textbf{1,890}& \textbf{1,731}& \textbf{2,543}& \textbf{3,120}& \textbf{4,087}\\ 
    			&
    			& Runtime
    			& \textbf{0.18}& \textbf{0.22}& \textbf{0.19}& \textbf{0.27}& \textbf{0.28}& \textbf{0.47}& \textbf{1.43}& \textbf{3.93}\\
    			&
    			& ASC
    			& 2.00& \textbf{1.73}& \textbf{1.72}& \textbf{1.76}& \textbf{1.67}& \textbf{1.91}& \textbf{1.72}& \textbf{1.62}\\ 
    			&
    			& F1 score
    			& 0.93& 0.94& 0.95& 0.95& \textbf{0.96}& \textbf{0.97}& \textbf{0.98}& \textbf{0.98}\\
    			&
    			& Precision
    			& 0.90& 0.92& 0.93& 0.93& 0.94& 0.95& 0.96& 0.96\\
    			&
    			& Recall
    			& \textbf{0.96}& \textbf{0.97}& \textbf{0.97}& \textbf{0.98}& \textbf{0.98}& \textbf{0.98}& \textbf{0.99}& \textbf{0.99}\\
    			\hline
    			& \multirow{6}{*}{PC}
    			& CI tests
    			& 2,577& 3,113& 4,586& 5,247& 6,094& 7,655& 12,868& 17,643
    			\\ 
    			&
    			& Runtime
    			& 0.32& 0.43& 0.71& 0.86& 1.16& 1.51& 2.94& 5.16
    			\\
    			&
    			& ASC
    			& \textbf{1.80}& 1.88& 2.16& 2.07& 2.15& 2.29& 2.52& 2.53
    			\\ 
    			&
    			& F1 score
    			& 0.92& 0.93& 0.95& 0.95& \textbf{0.96}& 0.96& \textbf{0.98}& \textbf{0.98}
    			\\
    			&
    			& Precision
    			& \textbf{1.00}& \textbf{1.00}& \textbf{1.00}& \textbf{1.00}& \textbf{1.00}& \textbf{1.00}& \textbf{1.00}& \textbf{1.00}
    			\\
    			&
    			& Recall
    			& 0.86& 0.87& 0.90& 0.91& 0.92& 0.93& 0.96& 0.97
    			\\
    			\hline
    			& \multirow{6}{*}{GS}
    			& CI tests
    			& 61,887& 102,296& NA& NA& NA& NA& NA& NA
    			\\ 
    			&
    			& Runtime
    			& 9.26& 16.12& NA& NA& NA& NA& NA& NA
    			\\
    			&
    			& ASC
    			& 5.79& 6.19& NA& NA& NA& NA& NA& NA
    			\\ 
    			&
    			& F1 score
    			& \textbf{0.94}& \textbf{0.95}& NA& NA& NA& NA& NA& NA
    			\\
    			&
    			& Precision
    			& \textbf{1.00}& \textbf{1.00}& NA& NA& NA& NA& NA& NA
    			\\
    			&
    			& Recall
    			& 0.89& 0.90& NA& NA& NA& NA& NA& NA
    			\\
    			\hline
    			& \multirow{6}{*}{CS}
    			& CI tests
    			& 14,091& 26,254& 27,131& 51,522& NA& NA& NA& NA
    			\\ 
    			&
    			& Runtime
    			& 2.47& 4.88& 5.45& 11.85& NA& NA& NA& NA
    			\\
    			&
    			& ASC
    			& 4.61& 5.03& 4.94& 5.53& NA& NA& NA& NA
    			\\ 
    			&
    			& F1 score
    			& \textbf{0.94}& 0.94& \textbf{0.96}& \textbf{0.96}& NA& NA& NA& NA
    			\\
    			&
    			& Precision
    			& 0.99& 0.99& \textbf{1.00}& \textbf{1.00}& NA& NA& NA& NA
    			\\
    			&
    			& Recall
    			& 0.89& 0.90& 0.92& 0.93& NA& NA& NA& NA
    			\\
    			\hline
    			& \multirow{6}{*}{MMPC}
    			& CI tests
    			& 2,818& 3,467& 4,675& 5,675& 6,412& 7,693& 12,747& 19,250
    			\\ 
    			&
    			& Runtime
    			& 0.28& 0.36& 0.51& 0.73& 0.94& 1.17& 2.50& 4.96
    			\\
    			&
    			& ASC
    			& 2.08& 2.13& 2.32& 2.33& 2.37& 2.46& 2.64& 2.78
    			\\ 
    			&
    			& F1 score
    			& 0.92& 0.93& 0.94& 0.95& \textbf{0.96}& 0.96& \textbf{0.98}& \textbf{0.98}
    			\\
    			&
    			& Precision
    			& 0.99& \textbf{1.00}&\textbf{1.00}& \textbf{1.00}& \textbf{1.00}& \textbf{1.00}& \textbf{1.00}& \textbf{1.00}
    			\\
    			&
    			& Recall
    			& 0.85& 0.87& 0.90& 0.90& 0.92& 0.93& 0.96& 0.97\\
    			\hline
    	    \end{tabular}
    	    \label{table: experiment}
        \end{table*}
        
        Table \ref{table: experiment} compares various algorithms on medium to large sized random graphs with $\Delta_\text{in}=4$, where $n=50p$ samples are available. 
        The entry NA indicates that the corresponding algorithm failed to learn a graph after performing $150,000$ CI tests on average. 
        Each number in the table is obtained using 20 DAGs. 
        In this table, ASC stands for the Average Size of Conditioning sets. 
        Moreover, precision and recall of the learned skeletons along with F1-scores are reported.
        
        The accuracy of all the algorithms is close to each other since the algorithms have access to a large dataset, whereas MARVEL is faster and requires significantly fewer CI tests with smaller conditioning sets compared to the other algorithms.
        
        \subsubsection{Scenario 3: real-world structures}
        Further experimental results are provided with new sets of parameters on real-world structures. 
        Two new structures, namely Alarm and Diabetes are added to the set of structures on which the structure learning algorithms are evaluated. 
        
        \begin{table*}[h]
    	    \caption{Scenario 3: Performance of the causal structure learning algorithms on real-world graphs after Markov boundary discovery $(n= 15p)$.}
    	    \fontsize{9}{10.5}\selectfont
    	    \centering
    	    \begin{tabular}{N|M{1.4cm}|M{1.3cm}||M{1.4cm}|M{1.25cm}|M{1.25cm}|M{1.25cm}|M{1.25cm}|M{1.6cm}|}
        		\hline
    			& \multicolumn{2}{c||}{\multirow{3}{*}{Algorithm}}
    			& Insurance
    			& Mildew
    			& Alarm
    			& Barley
    			& Win95pts
    			& Diabetes
    			 \\
    			& \multicolumn{2}{c||}{}
    			& p=27 & p=35 & p=37 & p=48 & p=76 & p=104 \\
    			& \multicolumn{2}{c||}{}
    			& e=51 & e=46 & e=46 & e=84 & e=70 & e=148 \\
    			\hline
    			& \multirow{6}{*}{MARVEL}
    			& CI tests
    			& \textbf{138}& \textbf{95}& \textbf{95}& \textbf{406}& \textbf{233}& \textbf{296}
    			\\
    			& 
    			& Runtime
    			& \textbf{0.01}& \textbf{0.01}& \textbf{0.01}& \textbf{0.03}& \textbf{0.03}& \textbf{0.07}
    			\\
    			& 
    			& ASC
    			& \textbf{0.91}& \textbf{0.51}& \textbf{0.78}& 1.14& \textbf{1.19}& \textbf{0.57}
    			\\
    			& 
    			& F1 Score
    			& \textbf{0.78}& \textbf{0.86}& \textbf{0.89}& \textbf{0.81}& \textbf{0.97}& \textbf{0.90}
    			\\
    			& 
    			& Precision
    			& 0.81& 0.89& 0.91& 0.81& 0.96& 0.90
    			\\
    			& 
    			& Recall
    			& \textbf{0.75}& \textbf{0.84}& \textbf{0.86}& \textbf{0.80}& \textbf{0.97}& \textbf{0.90}
    			\\
    			\hline
    			& \multirow{6}{*}{PC}
    			& CI tests
    			& 451& 403& 316& 870& 1,590& 1,584
    			\\
    			& 
    			& Runtime
    			& 0.03& 0.03& 0.03& 0.06& 0.16& 0.25
    			\\
    			& 
    			& ASC
    			& 1.02& 0.90& 0.95& \textbf{1.08}& 2.21& 1.25
    			\\
    			& 
    			& F1 Score
    			& 0.62& 0.69& 0.81& 0.73& 0.91& 0.84
    			\\
    			& 
    			& Precision
    			& 0.89& \textbf{0.96}& \textbf{1.00}& 0.96& \textbf{1.00}& 0.96
    			\\
    			& 
    			& Recall
    			& 0.48& 0.53& 0.68& 0.59& 0.84& 0.74
    			\\
    			\hline
    			& \multirow{6}{*}{GS}
    			& CI tests
                & 931& 1,417& 632& 3,046& 32,821& 1,825
    			\\
    			& 
    			& Runtime
    			& 0.04& 0.05& 0.03& 0.15& 3.29& 0.13
    			\\
    			& 
    			& ASC
    			& 2.15& 2.69& 2.03& 3.02& 6.38& 1.62
    			\\
    			& 
    			& F1 Score
    			& 0.70& 0.76& 0.86& 0.79& 0.94& 0.88
    			\\
    			& 
    			& Precision
    			& \textbf{0.91}& \textbf{0.96}& 0.99& \textbf{0.98}& \textbf{1.00}& \textbf{0.98}
    			\\
    			& 
    			& Recall
    			& 0.57& 0.63& 0.76& 0.66& 0.88& 0.80
    			\\
    			\hline
    			& \multirow{6}{*}{CS}
    			& CI tests
                & 219& 665& 140& 734& 1418& 445
    			\\
    			& 
    			& Runtime
    			& 0.03& 0.06& 0.03& 0.10& 0.30& 0.27
    			\\
    			& 
    			& ASC
    			& 1.56& 2.26& 1.45& 2.15& 3.39& 1.11
    			\\
    			& 
    			& F1 Score
    			& 0.71& 0.75& 0.86& 0.79& 0.93& 0.87
    			\\
    			& 
    			& Precision
    			& 0.90& 0.93& 0.97& 0.95& 0.98& 0.95
    			\\
    			& 
    			& Recall
    			& 0.60& 0.63& 0.77& 0.67& 0.89& 0.80
    			\\
    			\hline
    			& \multirow{6}{*}{MMPC}
    			& CI tests
    			& 574& 563& 416& 1,114& 1,481& 2,517
    			\\
    			& 
    			& Runtime
    			& 0.02& 0.02& 0.02& 0.04& 0.09& 0.17
    			\\
    			& 
    			& ASC
                & 1.39& 1.23& 1.09& 1.48& 2.14& 1.82
    			\\
    			& 
    			& F1 Score
                & 0.61& 0.65& 0.80& 0.71& 0.91& 0.81
    			\\
    			& 
    			& Precision
                & 0.88& 0.92& 0.99& 0.96& \textbf{1.00}& 0.88
    			\\
    			& 
    			& Recall
                & 0.47& 0.50& 0.67& 0.57& 0.84& 0.75
    			\\
    			\hline
    	    \end{tabular}
    	    \label{table: SM}
        \end{table*}
        
        Table \ref{table: SM} shows the experiment results of this scenario. 
        Each entry of the table is reported as an average of 20 runs, and $n=15p$ samples are available per variable.
        The data generation process is similar to the previous sections, except for the choice of coefficients and the variance of the noise variables.
        The coefficients are chosen uniformly at random from the interval $[-2,-0.5] \cup [0.5,2]$ and the noise variables are distributed according to $\mathcal{N}(0,\sigma_X^2)$ where $\sigma_X$ is chosen uniformly at random from the interval $[1, \sqrt{2}]$. 

        The experimental results demonstrate that MARVEL significantly outperforms the other algorithms in terms of runtime and the number of the required CI tests while maintaining superior accuracy of the learned graph in most of the experiments. 
        It is worthy to note that one of the caveats of constraint-based methods is the high number of missing edges in the learned structure, whereas MARVEL obtains the highest recall score (i.e., the fewest number of missing edges) as seen in our experimental results.

\section{Conclusion}
    We proposed MARVEL, a recursive Markov boundary-based causal structure learning method for efficiently learning the essential graph corresponding to the Markov equivalence class of the causal DAG. 
    We first introduced the notion of removable variables and then designed an efficient algorithm to identify them using Markov boundary information. 
    Then we made use of these variables to learn the causal structure recursively.
    We showed that MARVEL requires substantially fewer CI tests than the state-of-the-art methods, making it scalable and suitable to be used on systems with a large number of variables. 
    We provided the correctness and complexity analyses of the proposed method. 
    We also compared MARVEL with other constraint-based causal structure learning algorithms through various experiments. 
    The results demonstrated the superiority of MARVEL both in terms of complexity and accuracy compared to the other algorithms.
    



\clearpage
\appendix
\section{Proofs} \label{sec: appendix}
    We need the following definition in the proofs.
    \begin{definition}[Descendant]
        For vertices $X,Y$ in DAG $\mathcal{G}$, $Y$ is called a \emph{descendant} of $X$ if there is a directed path from $X$ to $Y$. The set of all descendants of $X$ in $\mathcal{G}$ is denoted by $\text{De}(X,\mathcal{G})$. Note that $X\in\text{De}(X,\mathcal{G})$.
    \end{definition}
    
\subsection{Proofs of Section \ref{sec: MARVEL}}
    {\bfseries Remark \ref{remark: removable}}
        Suppose $P_\mathbf{V}$ is Markov and faithful with respect to a DAG $\mathcal{G}$. 
        For any vertex $X\in \mathbf{V}$, $P_{\mathbf{V}\setminus\{X\}}$ is Markov and faithful with respect to $\mathcal{G}\setminus\{X\}$ if and only if $X$ is a removable vertex in $\mathcal{G}$. 
    \begin{proof}
        Suppose $Y,Z \in \mathbf{V}\setminus \{X\}$ and $\mathbf{S}\subseteq \mathbf{V} \setminus \{X,Y,Z\}$.
        By definition, $P_{\mathbf{V}\setminus\{X\}}$ is Markov and faithful with respect to $\mathcal{G}\setminus\{X\}$ if and only if 
        \begin{equation} \label{eq: proof remark removable 1}
            Y\perp_{\mathcal{G} \setminus \{X\}} Z\vert \mathbf{S} 
            \iff 
            Y\independent_{P_{\mathbf{V}\setminus \{X\}}} Z\vert \mathbf{S}.
        \end{equation}
        By the definition of removability, $X$ is removable in $\mathcal{G}$ if and only if 
        \begin{equation} \label{eq: proof remark removable 2}
            Y \perp_{\mathcal{G}\setminus \{X\}} Z \vert \mathbf{S}
		    \iff
		    Y \perp_{\mathcal{G}} Z \vert \mathbf{S}.
        \end{equation}
        Since $P_{\mathbf{V}}$ is Markov and faithful with respect to $\mathcal{G}$, the right hand sides of the above equations are equivalent.
        Hence, the two equations are equivalent. 
        
    \end{proof}
    {\bfseries Theorem \ref{thm: removablity} (Removability)}
        $X$ is removable in $\mathcal{G}$ if and only if the following two conditions are satisfied for every $Z\in\text{Ch}_{X}.$
		\begin{description}
			\item 
			    Condition 1: $N_X\subset N_Z\cup\{Z\}.$
			\item 
			    Condition 2: $\text{Pa}_Y \subset \text{Pa}_Z$ for any $Y\in\text{Ch}_X\cap \text{Pa}_Z$.
		\end{description}
	\begin{proof}
    	\begin{enumerate}
    	\item     
    	    To prove the if side, we assume that $X$ is a variable in $\mathcal{G}$ that satisfies Conditions 1 and 2. Let $\mathcal{H}=\mathcal{G}\setminus\{X\}$, the graph obtained by removing $X$ from $\mathcal{G}$. We first prove the following two lemmas. 
    		\begin{lemma}\label{lem: preserveDescendants}
    		    For any vertex $Y$ of $\mathcal{H}$, \[\text{De}(Y,\mathcal{G})=\text{De}(Y,\mathcal{H}).\]
            \end{lemma}
            \begin{proof}
                Suppose $Z\in \text{De}(Y,\mathcal{H})$, i.e., there exists a directed path from $Y$ to $Z$ in $\mathcal{H}$. The same path exists in $\mathcal{G}$. Hence, $Z \in \text{De}(Y,\mathcal{G})$. Therefore, $\text{De}(Y,\mathcal{H})\subseteq \text{De}(Y,\mathcal{G})$.
                
                Now suppose $Z\in \text{De}(Y,\mathcal{G})$, and let $\mathbf{P}$ be a directed path from $Y$ to $Z$ in $\mathcal{G}$. If $\mathbf{P}$ does not include $X$, it also exists in $\mathcal{H}$. Otherwise, let $\mathbf{P} = (Y=P_1,P_2,...,P_{i-1},P_i=X,P_{i+1},...,P_k=Z)$. Condition 1 implies that $P_{i-1}\to P_{i+1}$. Hence, $\mathbf{P}' = (Y=P_1,P_2,...,P_{i-1},P_{i+1},...,P_k=Z)$ is a directed path from $Y$ to $Z$ in $\mathcal{H}$, and $Z\in \text{De}(Y,\mathcal{H})$. Therefore, $\text{De}(Y,\mathcal{G})\subseteq \text{De}(Y,\mathcal{H})$, which concludes the proof.
            \end{proof}
            \begin{lemma} \label{lem: path block}
                Let $\mathbf{P}$ be a path in $\mathcal{H}$, and $\mathbf{S}$ be a subset of vertices of $\mathcal{H}$. $\mathbf{S}$ blocks $\mathbf{P}$ in $\mathcal{G}$ if and only if $\mathbf{S}$ blocks $\mathbf{P}$ in $\mathcal{H}$. Moreover, if a vertex blocks $\mathbf{P}$ in one of the graphs, it also blocks $\mathbf{P}$ in the other one.
            \end{lemma}
            \begin{proof}
                The proof of necessary and sufficient conditions are the same. Let $\mathcal{G}_1$ be one of $\mathcal{G},\mathcal{H}$ and $\mathcal{G}_2$ be the other one. Let $Y,Z$ be the end points of $\mathbf{P}$ and $W \in \mathbf{P}$ be a vertex that blocks $\mathbf{P}$ in $\mathcal{G}_1$. Either $W\in \mathbf{S}$ and $W$ is a non-collider in $\mathbf{P}$, or $W$ is a collider in $\mathbf{P}$ and $\text{De}(W,\mathcal{G}_1)\cap(\mathbf{S}\cup \{Y, Z\})=\varnothing$. Now consider this path in $\mathcal{G}_2$. In the first case, $W$ is a non-collider included in $\mathbf{S}$ and therefore it blocks $\mathbf{P}$ in $\mathcal{G}_2$ too. In the second case, $W$ is a collider and due to Lemma \ref{lem: preserveDescendants}, $\text{De}(W,\mathcal{G}_2)\cap(\mathbf{S}\cup \{Y, Z\})=$ $\text{De}(W,\mathcal{G}_1)\cap(\mathbf{S}\cup \{Y, Z\})=\varnothing$. Therefore, $W$ blocks $\mathbf{P}$ in $\mathcal{G}_2$.
            \end{proof}
    	    To show the removability of $X$,  we need to verify Equation \ref{eq: d-sepEquivalence}, i.e., show that for any vertices $Y,Z\in\mathbf{V}\setminus\{X\}$ and $\mathbf{S}\subseteq\mathbf{V}\setminus\{X,Y,Z\}$,
            \[
    		    Y \perp_{\mathcal{G}} Z \vert \mathbf{S}
    		    \iff
    		    Y \perp_{\mathcal{H}} Z \vert \mathbf{S}.
            \]
            Proving the only if side in Equation \ref{eq: d-sepEquivalence} is straightforward: Suppose $Y$ and $Z$ are d-separated in $\mathcal{G}$ by $\mathbf{S}$. All paths between $Y$ and $Z$ in $\mathcal{H}$, which are also present in $\mathcal{G}$, are blocked in $\mathcal{G}$ by $\mathbf{S}$. Lemma \ref{lem: path block} implies that these paths are also blocked in $\mathcal{H}$ by $\mathbf{S}$. Hence, $Y$ and $Z$ are d-separated in $\mathcal{H}$ by $\mathbf{S}$.
    		
            For the reverse direction, suppose $Y$ and $Z$ are d-separated in $\mathcal{H}$ by $\mathbf{S}$. Take an arbitrary path $\mathbf{P}$ between $Y$ and $Z$ in $\mathcal{G}$. We will prove that $\mathbf{P}$ is blocked in $\mathcal{G}$ by $\mathbf{S}$. If $X \notin \mathbf{P}$, then $\mathbf{P}$ is a path in $\mathcal{H}$ too. In this case, since $\mathbf{S}$ is blocking $\mathbf{P}$ in $\mathcal{H}$, Lemma \ref{lem: path block} implies that $\mathbf{S}$ blocks $\mathbf{P}$ in $\mathcal{G}$. Otherwise, $X \in \mathbf{P}$. Note that Lemma \ref{lem: path block} cannot be used in this case as $\mathbf{P}$ is not a path in $\mathcal{H}$. Suppose $\mathbf{P}=(P_1=Y, P_2,...,P_\ell,X,P_r,...,P_m=Z)$. Two possibilities may occur:
    	    \begin{enumerate}
    		    \item $P_\ell,P_r \in \text{Pa}_X$ (Figure \ref{fig:delete1}):
    		        If there exists a vertex other than $X$ that blocks $\mathbf{P}$ in $\mathcal{H}$, it blocks it in $\mathcal{G}$ too. Otherwise, we need to prove that $X$ is blocking $\mathbf{P}$ in $\mathcal{G}$. $X$ is a collider in $\mathbf{P}$ since $P_\ell,P_r \in \text{Pa}_X$. Note that $X \notin \mathbf{S}$ as $X$ is not present in $\mathcal{H}$. It is left to prove that $\text{De}(X,\mathcal{G})\cap (\mathbf{S}\cup \{ Y,Z\})=\varnothing$. Let $T$ be an arbitrary descendant of $X$ and $W\in \text{Ch}_X$ be the first vertex on a directed path from $X$ to $T$ (it might happen that $T=W$). As $X$ satisfies Condition 1, $P_\ell,P_r$ are connected to $W$. Now consider $\mathbf{P}' = (P_1, ..., P_\ell,W,P_r,...,P_m)$ which is a path between $Y$ and $Z$ in $\mathcal{H}$. $\mathbf{S}$ blocks this path in $\mathcal{H}$, but none of the vertices of $\mathbf{P}'$ except for $W$ can block $\mathbf{P}'$ in $\mathcal{H}$. This is because otherwise according to Lemma \ref{lem: path block}, the same vertex would block $\mathbf{P}$ in $\mathcal{G}$ too, which is against the assumption. Hence, $W$ blocks $\mathbf{P}'$. $W$ is a collider in $\mathbf{P}'$ and therefore, $\text{De}(W,\mathcal{G})\cap(\mathbf{S}\cup \{ Y,Z\})=\varnothing$. This proves that $T \notin (\mathbf{S}\cup \{ Y,Z\})$. As a result, $\mathbf{P}$ is blocked in $\mathcal{G}$ by $\mathbf{S}$.
    		    \item $P_\ell\in \text{Ch}_X$ or $P_r\in \text{Ch}_X$ (Figure \ref{fig:delete2}): 
    		        We assume without loss of generality that $P_r$ appears later than $P_\ell$ in the causal order. Therefore, $P_r\in \text{Ch}_X$. Due to Condition 1, $P_\ell\to P_r$ is an edge in $\mathcal{H}$. Hence, $\mathbf{P}'=(P_1,...,P_\ell,P_r,...,P_m)$ is a path between $Y$ and $Z$ in $\mathcal{H}$. This path is blocked in $\mathcal{H}$ (and also in $\mathcal{G}$ due to Lemma \ref{lem: path block}) by a vertex $W$. If $W\neq P_\ell$, or $W=P_\ell$ and $P_\ell$ is a non-collider in $\mathbf{P}$, $W$ blocks $\mathbf{P}$ in $\mathcal{G}$. The only remaining case to consider is when $P_\ell$ blocks $\mathbf{P}'$ but it is a collider in $\mathbf{P}$. In this case, $P_\ell \in \mathbf{S}$ since it is a non-collider in $\mathbf{P}'$. Moreover, $P_{\ell-1}$ is a parent of $P_\ell$ and due to Condition 2, $P_{\ell-1}\to P_r$ is an edge in $\mathcal{H}$. Hence, $\mathbf{P}''=(P_1,...,P_{\ell-1},P_r,...,P_m)$ is a path between $Y$ and $Z$ in $\mathcal{H}$. Now the same vertex that blocks $\mathbf{P}''$ in $\mathcal{H}$, blocks $\mathbf{P}$ in $\mathcal{G}$. Note that in this case $P_{\ell-1}$ is not a collider in either $\mathbf{P}''$ or $\mathbf{P}$.
    	    \end{enumerate} 
    	    In both cases $\mathbf{P}$ is blocked in $\mathcal{G}$ by $\mathbf{S}$. Hence, equation \ref{eq: d-sepEquivalence} holds and $X$ is removable. 
    	    
    		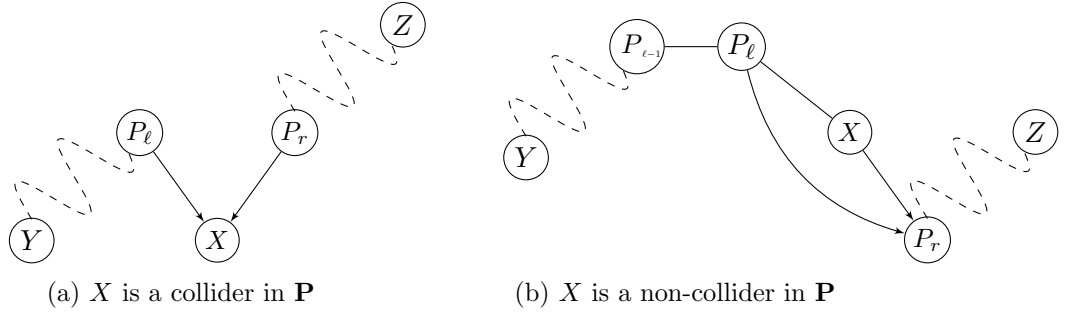
\begin{figure}[t]
    		    \centering
    		    \tikzstyle{block} = [draw, fill=white, circle, text width=1.1em, text centered, align=center]
        		\tikzstyle{input} = [coordinate]
        		\tikzstyle{output} = [coordinate]
    		    \begin{subfigure}[b]{0.3\textwidth}
        		    \centering
        		    \begin{tikzpicture}[->, node distance=1cm,>=latex', every node/.style={inner sep=1pt}]
        		        \node[block](X){\small $X$};
        		        \node[block](pl)[above left=1cm and 0.6cm of X]{\small $P_\ell$};
        		        \node[block](pr)[above right=1cm and 0.6cm of X]{\small $P_r$};
        		        \node[block](Y)[below left= 1 cm and 1cm of pl]{$Y$};
        		        \node[block](Z)[above right= 1 cm and 1cm of pr]{$Z$};
        		        \draw (pl) to (X);
        		        \draw (pr) to (X);
        		        \begin{scope}[dashed]
        		            \draw[-, rotate=45] ([xshift=0.2cm, yshift=0.2cm]Y) sin ([xshift=0.4cm, yshift=0.7cm]Y) cos  ([xshift=0.6cm, yshift=0.2cm]Y) sin
        		            ([xshift=0.8cm, yshift=-0.3cm]Y) cos
        		            ([xshift=1cm, yshift=0.2cm]Y) sin
        		            ([xshift=1.2cm, yshift=0.7cm]Y) cos
        		            ([xshift=1.4cm, yshift=0.2cm]Y) sin
        		            ([xshift=1.6cm, yshift=-0.3cm]Y) cos
        		            ([xshift=1.73cm, yshift=-0.1cm]Y);
        		            \draw[-, rotate=45] ([xshift=0.2cm, yshift=0.2cm]pr) sin ([xshift=0.4cm, yshift=0.7cm]pr) cos  ([xshift=0.6cm, yshift=0.2cm]pr) sin
        		            ([xshift=0.8cm, yshift=-0.3cm]pr) cos
        		            ([xshift=1cm, yshift=0.2cm]pr) sin
        		            ([xshift=1.2cm, yshift=0.7cm]pr) cos
        		            ([xshift=1.4cm, yshift=0.2cm]pr) sin
        		            ([xshift=1.6cm, yshift=-0.3cm]pr) cos
        		            ([xshift=1.73cm, yshift=-0.1cm]pr);
        		      \end{scope}
        		    \end{tikzpicture}
    		    \caption{$X$ is a collider in $\mathbf{P}$}
    		    \label{fig:delete1}
    		    \end{subfigure}\hspace{2cm}%
    		    \begin{subfigure}[b]{0.3\textwidth}
        		    \centering
        		    \begin{tikzpicture}[->, auto, node distance=1.3cm,>=latex', every node/.style={inner sep=1pt}]
        		        \node[block](X){\small $X$};
        		        \node[block](pr)[below right=1cm and 0.6cm of X]{\small $P_r$};
        		        \node[block](pl)[above left=0.7cm and 1cm of X]{$P_\ell$};
        		        \node[block, inner sep=2pt](plm)[left=0.7cm of pl]{\small $\scaleto{P}{0.7em}_{\scaleto{\ell-1}{0.3em}}$};
        		        \node[block](Y)[below left= 1 cm and 1cm of plm]{$Y$};
        		        \node[block](Z)[above right= 1 cm and 1cm of pr]{$Z$};
        		        \draw[-] (pl) to (X);
        		        \draw (X) to (pr);
        		        \draw (pl) to [bend right] (pr);
        		        \draw[-] (plm) to (pl);
        		        \begin{scope}[dashed]
        		            \draw[-, rotate=45] ([xshift=0.2cm, yshift=0.2cm]Y) sin ([xshift=0.4cm, yshift=0.7cm]Y) cos  ([xshift=0.6cm, yshift=0.2cm]Y) sin
        		            ([xshift=0.8cm, yshift=-0.3cm]Y) cos
        		            ([xshift=1cm, yshift=0.2cm]Y) sin
        		            ([xshift=1.2cm, yshift=0.7cm]Y) cos
        		            ([xshift=1.4cm, yshift=0.2cm]Y) sin
        		            ([xshift=1.6cm, yshift=-0.3cm]Y) cos
        		            ([xshift=1.73cm, yshift=-0.1cm]Y);
        		            \draw[-, rotate=45] ([xshift=0.2cm, yshift=0.2cm]pr) sin ([xshift=0.4cm, yshift=0.7cm]pr) cos  ([xshift=0.6cm, yshift=0.2cm]pr) sin
        		            ([xshift=0.8cm, yshift=-0.3cm]pr) cos
        		            ([xshift=1cm, yshift=0.2cm]pr) sin
        		            ([xshift=1.2cm, yshift=0.7cm]pr) cos
        		            ([xshift=1.4cm, yshift=0.2cm]pr) sin
        		            ([xshift=1.6cm, yshift=-0.3cm]pr) cos
        		            ([xshift=1.73cm, yshift=-0.1cm]pr);
        		        \end{scope}
        		    \end{tikzpicture}
    		    \caption{$X$ is a non-collider in $\mathbf{P}$}
    		    \label{fig:delete2}
    		    \end{subfigure}
    		    \caption{Omitting a removable vertex.}
    		    \label{fig:delete}
    		\end{figure}
    		
    	    \item To prove the only if side of Theorem \ref{thm: removablity}, it suffices to show that if $X$ is removable, it then satisfies conditions 1 and 2. 
    	    
    	    \emph{Condition 1}: 
    	        Suppose $Z\in \text{Ch}_{X}$ and $W\in N_{X}$. Either $W\gets X\to Z$ or $W\to X\to Z$ is a path in $\mathcal{G}$. A set $\mathbf{S}$ can block such paths only if $X\in\mathbf{S}$. Since $X\not\in \mathcal{H}$, no separating set for $W$ and $Z$ exists in $\mathcal{H}$. Equation \ref{eq: d-sepEquivalence} implies that no separating set for $W$ and $Z$ exists in $\mathcal{G}$. Therefore, $W,Z$ must be connected with an edge and $X$ satisfies Condition 1.
    	    
    	    \emph{Condition 2}: 
    	        Suppose $Z \in \text{Ch}_X$, $Y\in\text{Ch}_X \cap \text{Pa}_Z$, and $U \in \text{Pa}_Y$. We show that no $\mathbf{S}$ can d-separate $U$ and $Z$ in $\mathcal{G}$: $U\to Y\to Z$ and $U\to Y\gets X \to Z$ are both paths in $\mathcal{G}$. $\mathbf{S}$ can block the first path only if $Y\in\mathbf{S}$. But then the latter path can only be blocked if $X\in \mathbf{S}$. Since $X\not\in \mathcal{H}$, no such $\mathbf{S}$ exists in $\mathcal{H}$. Hence, Equation \ref{eq: d-sepEquivalence} implies that no separating set for $U$ and $Z$ exists in $\mathcal{G}$. Therefore, $U,Z$ must be connected with an edge. Since $U\to Y$ and $Y\to Z$, $U$ must be a parent of $Z$ and $X$ satisfies Condition 2.
    	\end{enumerate}
	\end{proof}
    {\bfseries Lemma \ref{lem: neighbor}}
        Suppose $X\in \mathbf{V}$ and $Y\in\Mb_X$. 
        $Y$ is a neighbor of $X$ if and only if 
        \begin{equation} \label{eq: neighbor proof}
            X\notindependent Y \vert \mathbf{S}, \hspace{ 0.5cm} \forall \mathbf{S}\subsetneq \Mb_X \setminus \{Y\}.
        \end{equation}
	\begin{proof}
	    If $Y \in N_X$, then $X,Y$ do not have any separating set. Otherwise, suppose $Y \notin N_X$. It suffices to find a $\mathbf{S} \subsetneq \Mb_X \setminus \{Y\}$ that d-separates $X,Y$. By local Markov property, if $Y$ is not a descendant of $X$, $\mathbf{S}=\textit{Pa}_X$ would do. Now suppose $Y$ is a descendant of $X$. Define $\mathbf{S}$ as the set of vertices in $\Mb_X$ that appear earlier than $Y$ in the causal order. We claim $X\independent Y\vert \mathbf{S}$, i.e., $\mathbf{S}$ blocks all the paths between $X$ and $Y$. Take an arbitrary path $\mathbf{P}$ between $X$ and $Y$ and let $Z\in \mathbf{P}$ be the latest vertex of $\mathbf{P}$ in the causal order. Two cases may occur:
	    \begin{enumerate}
	        \item $Z=Y$: In this case, all vertices on $\mathbf{P}$ appear earlier than $Y$ in the causal order and $\mathbf{P}\cap \Mb_X\subseteq \mathbf{S}$. Now let the two vertices following $X$ on $\mathbf{P}$ be $W_1,W_2$ ($\mathbf{P}$ has a length of at least 2 as $Y$ is not a neighbor of $X$). If $W_1$ is a parent of $X$, it is included in $\mathbf{S}$ and it blocks $\mathbf{P}$. Otherwise, $W_1$ is a child of $X$. Now either $W_2\in\text{Ch}_{W_1}$ and $W_1$ blocks $\mathbf{P}$ or $W_2\in\text{Pa}_{W_1}$ and $W_2$ blocks $\mathbf{P}$. Note that in the latter case $W_2$ is a parent of a child of $X$ and is included in $\Mb_X$. Therefore, $W_2\in\mathbf{S}$.
	        \item $Z\neq Y$: In this case $Z$ is a collider on $\mathbf{P}$ because both vertices before and after $Z$ appear earlier in the causal order and are therefore parents of $Z$. Due to the definition of $\mathbf{S}$, neither $Z$ nor any of its descendants are in $\mathbf{S}\cup \{X,Y\}$. Hence, $\mathbf{P}$ is blocked by $\mathbf{S}$.
	    \end{enumerate}
	    In all of the cases, the introduced $\mathbf{S}$ is not equal to $\Mb_X\setminus\{Y\}$ as it does not include the common child of $X$ and $Y$.
	\end{proof}
	{\bfseries Lemma \ref{lem: v-structures}}
	    Suppose $T \in \Lambda_X$ with a separating set $\mathbf{S}_{XT}$ for $X$ and $T$, and let $Y\in N_X$. 
        $Y$ is a common child of $X$ and $T$ (i.e., $X\to Y\gets T$ is in $\mathcal{V}_X^{\text{pa}}$)  if and only if $Y\notin \mathbf{S}_{XT}$ and
        \begin{equation}
            Y\notindependent T\vert \mathbf{S}, \hspace{0.5cm} \forall \mathbf{S}\subseteq \Mb_X\cup\{X\}\setminus\{Y,T\}.
        \end{equation}
	\begin{proof}
	    Suppose $Y\in N_X$ is a common child of $X$ and $T$. $\mathbf{S}_{XT}$ blocks the path $X \to Y \gets T$ between $X,T$. Hence, $Y$ cannot be in $\mathbf{S}_{XT}$. Additionally, $Y,T$ are neighbors and they do not have any separating set. 
	    
	    Now suppose $Y\in N_X$ is not a common child of $X$ and $T$, and $Y \notin \mathbf{S}_{XT}$. It suffices to find a $\mathbf{S}\subseteq \Mb_X\cup\{X\}\setminus\{Y,T\}$ that d-separates $Y$ and $T$. First, note that $Y$ and $T$ cannot be neighbors because otherwise, $(X,Y,T)$ is a path between $X,T$ that must be blocked by $\mathbf{S}_{XT}$, but $Y$ is not a collider in this path, and $Y \notin  \mathbf{S}_{XT}$ which is not possible. In order to introduce $\mathbf{S}$, consider three possible cases for $Y$:
	    \begin{itemize}
            \item{$Y\in \text{Pa}_X$:}
                We claim $\mathbf{S}=\mathbf{S}_{XT}$ d-separates $Y$ and $T$. Let $\mathbf{P}$ be a path between $Y$ and $T$. If $\mathbf{P}$ includes $X$, it is already blocked by $\mathbf{S}_{XT}$ as there exists a vertex that blocks the part of the path between $X$ and $T$, and the same vertex blocks $\mathbf{P}$. If $\mathbf{P}$ does not include $X$, let $\mathbf{P'}=(X,Y,\dots,T)$, i.e., the path from $X$ to $T$ through $Y$ and $\mathbf{P}$. $\mathbf{P'}$ is blocked by a vertex $Z$. As $Y\notin \mathbf{S}$ and $Y$ is not a collider in $\mathbf{P'}$, $Z\neq Y$. Therefore, $Z$ blocks $\mathbf{P}$ too.
            \item{$Y\in \text{Ch}_X$ and $T$ is a descendant of $Y$:}
                Similarly, $\mathbf{S}=\mathbf{S}_{XT}$ d-separates $Y$ and $T$. To prove this claim, let $\mathbf{P}$ be a path between $Y$ and $T$. If $\mathbf{P}$ includes $X$, with the same statements discussed above, $\mathbf{P}$ is blocked by $\mathbf{S}$. If $\mathbf{P}$ does not include $X$, define $\mathbf{P'}=(X,Y,\dots,T)$, and let $Z$ be the vertex that blocks $\mathbf{P'}$. $Z\neq Y$ because $Y\notin \mathbf{S}$ and $Y$ cannot block $\mathbf{P'}$ as a collider since $T$ is a descendant of $Y$. Hence, $Z$ blocks $\mathbf{P}$ too.
            \item{$Y\in \text{Ch}_X$ and $T$ is not a descendant of $Y$:}
                In this case, the set of parents of $Y$ d-separate $Y$ and $T$. Note that as $Y$ is a child of $X$, $\text{Pa}_Y\in\Mb_X\cup\{X\}$ and $\mathbf{S}=\text{Pa}_Y$ is our desired set.
        \end{itemize}
	    In all the cases we introduced a $\mathbf{S}\subseteq \Mb_X\cup\{X\}\setminus\{Y,T\}$ that d-separates $Y$ and $T$.
	\end{proof}
	{\bfseries Lemma \ref{lem: condition1}}
	    Variable $X$ satisfies Condition 1 of Theorem \ref{thm: removablity} if and only if 
		\begin{equation} \label{eq: condition1 proof}
		    Z \notindependent W\vert \mathbf{S} \cup \{X\}, \hspace{0.5cm} \forall W,Z\in N_X,\, \mathbf{S} \subseteq \Mb_X \setminus \{Z, W\}.
		\end{equation}
	\begin{proof}
	        Suppose $X$ satisfies condition 1, and let $Z,W$ be two of its neighbors. If at least one of $Z,W$ is a child of $X$, then condition 1 implies that $Z$ and $W$ are neighbors and cannot be d-separated with any set. If both of $Z,W$ are parents of $X$, no set including $X$ can d-separates $Z,W$ since $X$ is a collider in $Z \to X \gets W$. 
	        
	        For the if side, suppose that $Z \notindependent W\vert \mathbf{S} \cup \{X\}$ for any pair of vertices $W,Z\in N_X$ and any $\mathbf{S} \subseteq \Mb_X \setminus \{Z, W\}$. We need to show that $X$ satisfies condition 1. Let $Z,W$ be two arbitrary vertices in $N_X$ where at least one of them is a child of $X$. It suffices to show that $W$ and $Z$ are neighbors. Without loss of generality we can assume that $Z$ appears later than $W$ in the causal order. Therefore $Z\in \text{Ch}_X$. Assume by contradiction that they are not neighbors. Since $W$ is not a descendant of $Z$, local Markov property implies that the set of parents of $Z$ d-separates $W$ and $Z$. Note that $Z$ is a child of $X$ and therefore $\text{pa}_Z\subseteq \Mb_X\cup\{X\}$. Hence, $\mathbf{S}=\text{Pa}_Z\setminus\{X\} \subseteq \Mb_X \setminus \{Z, W\}$ would be a set that $Z \independent W\vert \mathbf{S} \cup \{X\}$ which is against our assumption. Therefore, $W$ and $Z$ are neighbors and $X$ satisfies condition 1.
	\end{proof}
	{\bfseries Lemma \ref{lem: condition2}}
	    Suppose the variable $X$ satisfies Condition 1 of Theorem \ref{thm: removablity}. Then $X$ satisfies Condition 2 of Theorem \ref{thm: removablity}, and therefore, $X$ is removable, if and only if 
	            \begin{equation} \label{eq: condition2 proof}
	                Z \notindependent T\vert \mathbf{S} \cup \{X, Y\}, \hspace{0.5cm} \forall (X\to Y\gets T) \in \mathcal{V}_X^{\text{pa}}, \, Z\in N_X \setminus \{Y\},\, \mathbf{S} \subseteq \Mb_X \setminus \{Z, Y,  T\}.
	            \end{equation}
	\begin{proof}
	    Suppose $X$ is removable. Let $X\to Y \gets T$ be a v-structure and $Z\in N_X$. If $Z$ appears later than $Y$ in the causal order, then condition 1 implies that $Y \in \text{Pa}_Z$. Therefore, condition 2 implies that $T\in \text{Pa}_Z$. Hence, $Z$ and $T$ are neighbors and cannot be d-separated. If $Z$ appears earlier than $Y$ in the causal order, condition 1 implies that $Z\in \text{Pa}_Y$. In this case $Z\to Y\gets T$ is a path between $Z,T$ and no set containing $Y$ can d-separate $Z,T$. 
	    
	    For the if side, suppose $X$ satisfies condition 1 and the assumption of the lemma holds. We need to show that $X$ satisfies condition 2. Suppose $Z\in \text{Ch}_X$, $Y\in\text{Ch}_X\cap \text{Pa}_Z$, and $T\in \text{Pa}_Y$. It suffices to show that $T,Z$ are neighbors. Assume by contradiction that $Z,T$ are not neighbors. Note that $X,Y \in \text{Pa}_Z$. Define $\mathbf{S}= \text{Pa}_Z \setminus \{X,Y\}$. Parents of $Z$ are in $\Mb_X$, and $Z,T$ are not neighbors. Therefore, $\mathbf{S} \subseteq \Mb_X \setminus \{Z,Y,T\}$. Since $T$ is not a descendant of $Z$ (in fact, $Z$ is a descendant of $T$), parents of $Z$ d-separates $Z,T$. Hence, $\mathbf{S}\cup \{X,Y\}$ d-separates $Z,T$ which is against the assumption of the lemma. Hence, $Z,T$ are neighbors and $X$ is removable.
	\end{proof}
	{\bfseries Lemma \ref{lem: MbboundForremovable}}
	    If $X\in\mathbf{V}$ is a removable vertex in $\mathcal{G}$, then $\left\vert \Mb_X\right\vert\leq \Delta_{\text{in}}$, where $\Delta_{\text{in}}$ is the maximum in-degree of $\mathcal{G}$.
	\begin{proof}
        Let $Z$ be the latest child of $X$ and therefore, the latest vertex of $\Mb_X$ in the causal order. From Theorem \ref{thm: removablity}, every vertex in $\Mb_X\setminus\{Z\}$ is connected to $Z$. By definition of $Z$, vertices in $\Mb_X\setminus\{Z\}$ along with $X$ itself must be the parents of $Z$, that is $\{X\} \cup \Mb_X\setminus\{Z\} = \text{Pa}_Z$. The cardinality of the left hand side is equal to $\left\vert\Mb_X\right\vert$, while the cardinality of the right hand side is bounded by $\Delta_{in}$. Hence, \[\left\vert\Mb_X\right\vert=\left\vert \text{Pa}_Z\right\vert\leq\Delta_\text{in}.\]
    \end{proof}
    {\bfseries Theorem \ref{thm: correctness} (Correctness of MARVEL)}
        Suppose $\mathcal{G}$ satisfies Markov property and faithfulness with respect to $P_{\mathbf{V}}$.
	    The learned graph $\hat{\mathcal{G}}$ in Algorithm \ref{MARVEL} has the same skeleton and v-structures as $\mathcal{G}$. 
	    Therefore, the output of Algorithm \ref{MARVEL} is the essential graph corresponding to $\mathcal{G}$.
	\begin{proof}
        First, note that as it is discussed in Section \ref{sec: algo}, Markov and faithfulness assumptions hold in all iterations of Algorithm \ref{MARVEL}. Hence, throughout the execution of the algorithm, the result of CI tests are equivalent to the d-separation relations in the remaining graph. Further note that Remark \ref{rem: removables exist} implies that there always exists at least one removable variable at each iteration. Therefore, Algorithm \ref{MARVEL} never gets stuck. We now show that $\hat{\mathcal{G}}$ has the same skeleton and the same set of v-structures as $\mathcal{G}$. 
	    \begin{itemize}
	        \item \emph{Skeleton}:
	            We show $\hat{\mathcal{G}}$ contains all the edges of $\mathcal{G}$ and it has no extra edges. 
	            
	            \textit{No false positives}:
    	            The algorithm starts with an empty graph. At each iteration, new edges are only added between $X_{(i)}$ and its neighbors (line 8). Hence, no extra edges appear in $\hat{\mathcal{G}}$.
    
        	    \textit{No false negatives}:
                    Suppose there exists an edge between $X$ and $Y$ in $\mathcal{G}$. Without loss of generality, suppose $X$ gets removed before $Y$. During the iteration at which the algorithm removes $X$, it identifies all the remaining neighbors of $X$, including $Y$, and adds an edge between $X$ and its neighbors. Therefore, the skeleton of $\hat{\mathcal{G}}$ contains the edge between $X$ and $Y$. 
	        \item \emph{V-structures}:
        	    It suffices to show that for every $X,Y$ and $Z$,
        	    \[X\to Y\leftarrow Z \text{ is a v-structure in } \mathcal{G} \text{ if and only if } X\to Y\leftarrow Z \text{ is a v-structure in } \hat{\mathcal{G}}.\]
        	    
	            If side: Suppose $X\to Y\gets Z$ is a v-structure in $\hat{\mathcal{G}}$. Consider the two corresponding edges from $X$ to $Y$ and $Z$ to $Y$. Each edge is either oriented in line 11 of the algorithm, where it is a part of a v-structure. In this case, it is oriented correctly due to Lemma \ref{lem: v-structures}. Otherwise, it is oriented in line 13 when vertex $Y$ gets removed. Since $Y$ is getting removed, it is removable while both $X$ and $Z$ are present in the remaining graph. Therefore, $Y$ satisfies condition 1 of Theorem \ref{thm: removablity}. We know from the first part of the proof that $\hat{\mathcal{G}}$ has the same skeleton as $\mathcal{G}$. Hence, $X$ and $Z$ are not neighbors in $\mathcal{G}$. Neither $X$ nor $Z$ can be a child of $Y$ in $\mathcal{G}$, since otherwise, condition 1 implies that $X$ and $Z$ are neighbors. Therefore, they must both be parents of $Y$ in $\mathcal{G}$. That is, $X\to Y\gets Z$ is a v-structure in $\mathcal{G}$.
	    
	            Only if side: Suppose $X\to Y\gets Z$ is a v-structure in $\mathcal{G}$. We show that the edge between $X$ and $Y$ is oriented from $X$ to $Y$ in $\hat{\mathcal{G}}$. We orient every edge either in line 11 where this edge appears in an identified v-structure or in line 13 when the edge is undirected and we want to remove one of its endpoints. Note that all the edges are oriented during the algorithm. In the first case, the orientation of the edge is correct due to Lemma \ref{lem: v-structures}. In the latter case, i.e., if the edge between $X$ and $Y$ is oriented in line 13, one of the endpoints is discovered to be removable. It suffices to show that the removable endpoint is $Y$, and hence, the edge is oriented from $X$ to $Y$. Suppose the opposite, that is $X$ is discovered to be removable and the algorithm reaches line 13. In this case, the v-structure $X\to Y\gets Z$ must be identified in line 10 as a member of $\mathcal{V}_X^{\text{pa}}$, and the edge between $X$ and $Y$ is oriented in line 11, which contradicts the assumption that this edge is not oriented until line 13. Therefore, the removable endpoint is $Y$, and while $Y$ is removed, the undirected edge is oriented correctly from $X$ to $Y$. 
	            
	            In both cases, the edge is from $X$ to $Y$. The same arguments hold for the edge between $Z$ and $Y$. Therefore, $X\to Y\leftarrow Z$ is a v-structure in $\hat{\mathcal{G}}$. 
	    \end{itemize}
	\end{proof}
\subsection{Proofs of Section \ref{sec: complexity}}
	{\bfseries Proposition \ref{prp: complexity}}
	    Given the initial Markov boundaries, the number of CI tests required by Algorithm \ref{MARVEL} on a graph of order $p$ and maximum in-degree $\Delta_{in}$ is upper bounded by
		\begin{equation}\label{eq: complexity proof}
		    p \binom{\Delta_{\text{in}}}{2} + \frac{p}{2}\Delta_{\text{in}}(1+ 0.45\Delta_{\text{in}} )2^{\Delta_{\text{in}}}  = \mathcal{O}(p\Delta_{in}^22^{\Delta_{in}}).
		\end{equation}
	\begin{proof}
	    MARVEL performs CI tests for the following purposes:
	    \begin{enumerate}
	        \item
	            \textit{Updating Markov boundaries at the end of each iteration}: As discussed in Section \ref{sec: update Mb}, when $X$ is removed, it is enough to perform a CI test for each pair $(Y,Z)\in N_X\times N_X$. There are $\binom{|N_X|}{2}$ such pairs and from Lemma \ref{lem: MbboundForremovable} we know $|N_X|\leq|\Mb_X|\leq \Delta_\text{in}$. Hence, at most $p \binom{\Delta_\text{in}}{2}$ CI tests are performed for updating the Markov boundaries throughout the algorithm. 
	        \item
	             \textit{Testing for removability}: As discussed in Section \ref{sec: testing removability}, given Markov boundary information, for each variable $X$ we can test its removability by first finding $N_X,\Lambda_X$, then finding $\mathcal{V}_X^{\text{pa}}$, and then checking Conditions 1 and 2 of Theorem \ref{thm: removablity}. We showed that we can do this by performing at most 
	             \[K = |\Mb_X| 2^{|\Mb_X|-1} + |\Lambda_X||N_X|2^{|\Mb_X|-1} + \binom{|N_X|}{2} 2^{|\Mb_X|-2} + \vert\Lambda_X\vert \vert N_X\vert 2^{|\Mb_X|-2} \]
	             \[ = 2^{|\Mb_X|-2}(2|\Mb_X| + 3|\Lambda_X||N_X| + \binom{|N_X|}{2} ) \]
	             CI tests. From Lemma \ref{lem: MbboundForremovable} we know $|\Mb_X|\leq \Delta_\text{in}$. Suppose $N = |N_X|$. Since $|\Mb_X| =|N_X|+|\Lambda_X|$ we have:
	             \[K < 2^{2\Delta_{\text{in}}-2} (2\Delta_\text{in} + 3 (\Delta_\text{in} -N)N + \frac{N^2}{2} ) \leq 2^{\Delta_{\text{in}}-2} (2\Delta_{\text{in}}+ 0.9\Delta_{\text{in}}^2 )= \frac{1}{2}\Delta_{\text{in}}(1+ 0.45\Delta_{\text{in}} )2^{\Delta_{\text{in}}} \]
	             As we discussed in Section \ref{sec: save}, we only need to perform CI tests for testing the removability of a variable for the first time. Because we can save some information and avoid performing new CI tests in the next iterations. Hence, at most $\frac{p}{2}\Delta_{\text{in}}(1+ 0.45\Delta_{\text{in}} )2^{\Delta_{\text{in}}}$ CI tests are performed for this task.
	    \end{enumerate}
	    Summing up over the above bounds, we get the desired upper bound.
	    
	\end{proof}
	{\bfseries Theorem \ref{thm: lwrBound}}
	    The number of conditional independence tests of the form $X \independent Y\vert \mathbf{S}$ required by any constraint-based algorithm on a graph of order $p$ and maximum in-degree $\Delta_{in}$ in the worst case is lower bounded by
		\begin{equation} \label{eq: lwrbound proof}
		    \Omega(p^2+p\Delta_{in}2^{\Delta_{in}}).
		\end{equation}
	\begin{proof}
	    We first present an example that requires at least
	    \begin{equation} \label{eq: lwrbound proof1}
	        \floor{\frac{p}{\Delta_{\text{in}}+1}} \binom{\Delta_{\text{in}}+1}{2}2^{\Delta_{\text{in}}-1} =  \Omega(p\Delta_{in}2^{\Delta_{in}})
	    \end{equation}
	    CI tests to uniquely find its skeleton.  
	    
	    {\bfseries Example 1.}
        Let $d=\Delta_{\text{in}}$. Suppose the variables are denoted by $X_1,X_2,...,X_p$ and the first $(d+1)\floor{\frac{p}{d+1}}$ variables are partitioned into $\floor{\frac{p}{d+1}}$ clusters $C_1,...,C_{\floor{\frac{p}{d+1}}}$ each of size $(d+1)$. Let $\mathcal{G}$, the causal graph, have the following structure: $X_i$ is the $i$-th vertex in the causal order, the induced sub-graph over the vertices of each cluster is a complete graph, and there is no edge between vertices of different clusters. Note that the maximum in-degree of $\mathcal{G}$ is $d$. Given any algorithm $\mathcal{A}$ that performs fewer CI tests than the claimed lower bound in Equation \ref{eq: lwrbound proof1}, we provide a graph $\mathcal{H}$ such that $\mathcal{A}$ fails to tell $\mathcal{G}$ and $\mathcal{H}$ apart. 
        
        Considering the structure of $\mathcal{G}$, for an arbitrary $\mathbf{S}$, any CI test of $X_i$ and $X_j$ conditioned on $\mathbf{S}$ that $\mathcal{A}$ queries, yields dependence if $X_i$ and $X_j$ are in the same cluster and yields independence otherwise. There are $M=\floor{\frac{p}{d+1}}\binom{d+1}{2}$ pairs $\{X_i,X_j\}$ such that $X_i$ and $X_j$ are in the same cluster. As algorithm $\mathcal{A}$ performs less than $M2^{d-1}$ CI tests, there exists a pair $\{X_{i^*},X_{j^*}\}$ in a particular cluster for which algorithm $\mathcal{A}$ queries the conditional independence of $\{X_{i^*}$ and $X_{j^*}\}$ conditioned on fewer than $2^{d-1}$ sets. Without loss of generality, suppose $i^*<j^*\leq d+1$ and the corresponding cluster is $C_1=\{X_1,...,X_{d+1}\}$. As $C_1 \setminus \{X_{i^*}, X_{j^*}\}$ has $2^{d-1}$ subsets, there exists at least one subset $\mathbf{S}^* \subseteq C_1 \setminus \{X_{i^*},X_{j^*}\}$ such that for no $\mathbf{S}' \subseteq \mathbf{V}\setminus C_1$, algorithm $\mathcal{A}$ queries the result of the CI test $X_{i^*} \independent X_{j^*}\vert \mathbf{S}$ where $\mathbf{S}= \mathbf{S}^*\cup \mathbf{S}'$.
        
        \begin{figure}[h]
        	\centering
        	\tikzstyle{block} = [circle, inner sep=1.2pt, fill=black]
        	\begin{tikzpicture}[->]
        		\node[block](1){};
        		\node[block](2)[below left=0.6cm and 0.6cm of 1]{};
        		\node[block](3)[below right=0.5cm and 1.5cm of 2]{};
        		\node[block](4)[below left=1cm and 0.6cm of 2]{};
        		\draw (1) to (2);
        		\draw (1) to (3);
        		\draw (2) to (3);
        		\draw (1) to [bend left=15] (4);
        		\draw (2) to (4);
        		\draw (3) to (4);
        		\node[block,red](xi)[below right= 1.1cm and 0.2 cm of 4]{};
        		\node[block,red](xj)[below right= 0.5cm and 1.4 cm of xi]{};
        		\draw[dashed, red, line width=0.4mm] (xi) to (xj);
        		\draw (1) to [bend left=15] (xi);
        		\draw (2) to [bend left=10] (xi);
        		\draw (2) to [bend left=10] (xj);
        		\draw (3) to [bend left = 20](xi);
        		\draw (4) to (xi);
        		\draw (4) to [bend left=10](xj);
        		\draw (3) to [bend left=5](xj);
        		\draw (1) to [bend left=15](xj);
        		\node[block](5)[below left= 1cm and 0.7cm of xj]{};
        		\node[block](6)[below right= 0.6cm and 0.7cm of 5]{};
        		\node[block](7)[below left= 1cm and 0.3cm of 5]{};
        		\draw (5) to (6);
    			\draw (5) to (7);
    			\draw (6) to (7);
        		\draw[rounded corners,dashed,blue] ([xshift=-0.3cm, yshift=-0.4cm]4) rectangle (1.2cm,0.3cm);
        		\node[blue] (s)[left=1.6cm of 1]{$\mathbf{S}^*$};
        		\node (xi name)[left=0.05cm of xi]{$X_{i^*}$};
        		\node (xi name)[below right=-2mm and 0.1mm of xj]{$X_{j^*}$};
        		\draw (xi) to [bend right=15](5);
        		\draw (xi) to [bend left=10](6);
        		\draw (xi) to [bend right=10](7);
        		\draw (xj) to [bend left=8](5);
        		\draw (xj) to [bend left=10](6);
        		\draw (xj) to [bend left=18](7);
        		\draw[-] (1) to [bend left=5]([xshift=1cm, yshift=-0.7cm]1);
        		\draw[-] (1) to [bend left=5]([xshift=1cm, yshift=-0.55cm]1);
        		\draw[-] (1) to [bend left=5]([xshift=1cm, yshift=-0.4cm]1);
    			\node[rotate=-45, below right= 0.5 cm and 1cm of 1]{\ldots};
    			\node[rotate=-45, below right= 0.35 cm and 1cm of 1]{\ldots};
    			\node[rotate=-45, below right= 0.2 cm and 1cm of 1]{\ldots};
        		\draw[-] (3) to [bend left=5]([xshift=0.7cm, yshift=-0.7cm]3);
    			\draw[-] (3) to [bend left=5]([xshift=0.7cm, yshift=-0.55cm]3);
    			\draw[-] (3) to [bend left=5]([xshift=0.7cm, yshift=-0.4cm]3);
    			\node[rotate=-45, below right= 0.5 cm and 0.7cm of 3]{\ldots};
    			\node[rotate=-45, below right= 0.35 cm and 0.7cm of 3]{\ldots};
    			\node[rotate=-45, below right= 0.2 cm and 0.7cm of 3]{\ldots};
        		\draw[-] (2) to [bend right=5]([xshift=-0.8cm, yshift=-0.55cm]2);
    			\draw[-] (2) to [bend right=5]([xshift=-0.8cm, yshift=-0.4cm]2);
    			\draw[-] (2) to [bend right=5]([xshift=-0.8cm, yshift=-0.25cm]2);
    			\node[rotate=45, below left= 0.3 cm and 0.8cm of 2]{\ldots};
    			\node[rotate=45, below left= 0.15 cm and 0.8cm of 2]{\ldots};
    			\node[rotate=45, below left= 0.0 cm and 0.8cm of 2]{\ldots};
        		\draw[-] (5) to [bend left=5]([xshift=-1.1cm, yshift=0.6cm]5);
    			\draw[-] (5) to [bend left=5]([xshift=-1.1cm, yshift=0.45cm]5);
    			\draw[-] (5) to [bend left=5]([xshift=-1.1cm, yshift=0.3cm]5);
    			\draw[-] (5) to [bend left=5]([xshift=-1.1cm, yshift=0.15cm]5);
    			\node[rotate=-40, above left= 0.37 cm and 1.05cm of 5]{\ldots};
    			\node[rotate=-35, above left= 0.2 cm and 1.02cm of 5]{\ldots};
    			\node[rotate=-26, above left= 0.05 cm and 1cm of 5]{\ldots};
    			\node[rotate=-20, above left= -0.07 cm and 1cm of 5]{\ldots};
        		\draw[-] (6) to [bend right=5]([xshift=0.7cm, yshift=0.6cm]6);
    			\draw[-] (6) to [bend right=5]([xshift=0.7cm, yshift=0.45cm]6);
    			\draw[-] (6) to [bend right=5]([xshift=0.7cm, yshift=0.3cm]6);
    			\draw[-] (6) to [bend right=5]([xshift=0.7cm, yshift=0.15cm]6);
    			\node[rotate=45, above right= 0.35 cm and 0.7cm of 6]{\ldots};
    			\node[rotate=45, above right= 0.18 cm and 0.7cm of 6]{\ldots};
    			\node[rotate=42, above right= 0.05 cm and 0.7cm of 6]{\ldots};
    			\node[rotate=38, above right= -0.1 cm and 0.7cm of 6]{\ldots};
    			\draw[-] (7) to [bend left=5]([xshift=-0.7cm, yshift=0.6cm]7);
    			\draw[-] (7) to [bend left=5]([xshift=-0.7cm, yshift=0.45cm]7);
    			\draw[-] (7) to [bend left=5]([xshift=-0.7cm, yshift=0.3cm]7);
    			\draw[-] (7) to [bend left=5]([xshift=-0.7cm, yshift=0.15cm]7);
    			\node[rotate=-45, above left= 0.38 cm and 0.66cm of 7]{\ldots};
    			\node[rotate=-45, above left= 0.18 cm and 0.65cm of 7]{\ldots};
    			\node[rotate=-40, above left= 0.03 cm and 0.63cm of 7]{\ldots};
    			\node[rotate=-30, above left= -0.1 cm and 0.6cm of 7]{\ldots};
        	\end{tikzpicture}
        	\caption{Cluster $C_1$ in the graph $\mathcal{H}$}
        	\label{fig: prp5}
        \end{figure}
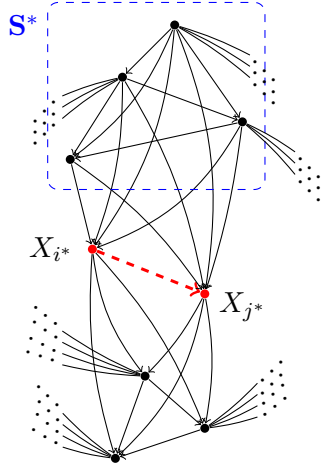
        Now we provide the graph $\mathcal{H}$ as follows. It has the same structure and causal order as $\mathcal{G}$, except over the vertices of $C_1$. As depicted in Figure \ref{fig: prp5}, the vertices in $\mathbf{S}^*$ in an arbitrary order, form the first vertices of $C_1$ in the causal order. These vertices are followed by ${X_i^*}$ and ${X_j^*}$ and then the rest of the vertices of $C_1$, again in an arbitrary order. As for the skeleton, all pairs in $C_1$ are connected to each other except for $({X_i^*},{X_j^*})$. 
        
        It suffices to show that $\mathcal{H}$ is consistent with all the CI tests that $\mathcal{A}$ performs. For an arbitrary CI test between $X_i$ and $X_j$ conditioned on $\mathbf{S}$, if $\{X_i,X_j \} \neq \{X_{i^*},X_{j^*}\}$ the test yields the same result for $\mathcal{G}$ and $\mathcal{H}$. If $\{X_i,X_j \} = \{X_{i^*},X_{j^*}\}$, the result always yields dependence as there is an edge between ${X_i^*}$ and ${X_j^*}$ in $\mathcal{G}$. To prove the consistency, we have to show that none of the sets $\mathbf{S}$ among the conditioning sets of the CI tests that $\mathcal{A}$ performs, d-separates ${X_i^*}$ and ${X_j^*}$. The structure of $\mathcal{H}$ implies that $\mathbf{S}^*$ is the unique subset of $C_1$ which d-separates ${X_i^*}$ and ${X_j^*}$. Since there is no edge between vertices of different clusters, if $\mathbf{S}$ d-separates ${X_i^*}$ and ${X_j^*}$ in $\mathcal{H}$ then $\mathbf{S} \cap C_1$ must be equal to $\mathbf{S}^*$. As mentioned above, there is no such CI test performed in $\mathcal{A}$. Therefore, $\mathcal{H}$ is consistent with the results of the CI tests, and $\mathcal{A}$ cannot uniquely determine the skeleton of $\mathcal{G}$.
        
        We now provide another example that requires $\Omega(p^2)$ CI tests to uniquely find its skeleton.
        
        {\bfseries Example 2.}
        Suppose the causal graph $\mathcal{G}$ is an empty graph with $p$ vertices.
        Given any algorithm $\mathcal{A}$ that performs fewer CI tests than $\binom{p}{2}$, we provide a graph $\mathcal{H}$ such that $\mathcal{A}$ fails to tell $\mathcal{G}$ and $\mathcal{H}$ apart. 
        
        As algorithm $\mathcal{A}$ performs less than $\binom{p}{2}$ CI tests, there exists a pair $\{X_{i^*},X_{j^*}\}$ for which algorithm $\mathcal{A}$ does not query the conditional independence of $X_{i^*}$ and $X_{j^*}$ conditioned any set.
        Let $\mathcal{H}$ be the graph with $p$ vertices and only one edge $(X_{i^*},X_{j^*})$. 
        Note that, all of the performed CI tests yield independence since $\mathcal{G}$ is an empty graph. 
        Therefore, $\mathcal{H}$ is consistent with the results of the CI tests, and $\mathcal{A}$ cannot uniquely determine the skeleton of $\mathcal{G}$.
        
        Combining the lower bounds of the above examples, we get the desired lower bound of Equation \ref{eq: lwrbound proof}.

	\end{proof}

\clearpage
\bibliography{bibliography}

\end{document}